\documentclass[letter,11pt]{article}
\usepackage[margin=1in]{geometry}
\usepackage{amsthm}
\usepackage{amsfonts}
\usepackage{amsmath,graphicx,multicol}
\usepackage{amssymb}
\usepackage[utf8]{inputenc}
\usepackage[english]{babel}
\usepackage{bm}
\usepackage{epstopdf}
\usepackage{hyperref}
\usepackage{enumitem}
\usepackage{algorithm,algorithmic}
\usepackage{cleveref}
\usepackage{xcolor}
\usepackage{subcaption}
\usepackage{url}
\newcommand{\E}{\mathbb{E}}
\newcommand{\R}{\mathbb{R}}
\usepackage{cite}
\usepackage{booktabs} 
\usepackage{mathtools}

\newtheorem{definition}{Definition}[]
\newtheorem{assumption}[]{Assumption}
\newtheorem{remark}{\textbf{Remark}}

\newtheorem{theorem}{Theorem}[]

\newtheorem{lemma}[]{Lemma}
\usepackage{multirow}
\usepackage{makecell}   
\usepackage{colortbl}   
\usepackage{hhline}
\definecolor{tabcolor}{RGB}{245,237,227} 

\title{\vspace{-10mm}\textbf{Lower Bounds and Proximally Anchored SGD for Non-Convex Minimization Under Unbounded Variance}}
\date{}
\vspace{-0.05in}
\author{
Arda Fazla$^\dagger$, Ege C. Kaya$^\dagger$, Antesh Upadhyay$^\dagger$, Abolfazl Hashemi\thanks{Authors with the School of Electrical and Computer Engineering, Purdue University, West Lafayette, IN 47907, USA. $^\dagger$ denotes equal contribution.}}
\begin{document}
\maketitle
\vspace{-5mm}
\begin{abstract}
\vspace{-7mm}
\noindent{
}
\\\\\noindent
Analysis of Stochastic Gradient Descent (SGD) and its variants typically relies on the assumption of uniformly bounded variance, a condition that frequently fails in practical non-convex settings, such as neural network training, as well as in several elementary optimization settings. While several relaxations are explored in the literature, the Blum-Gladyshev (BG-0) condition, which permits the variance to grow quadratically with distance has recently been shown to be the weakest condition. However, the study of the oracle complexity of stochastic first-order non-convex optimization under BG-0 has remained underexplored. In this paper, we address this gap and establish information-theoretic lower bounds, proving that finding an $\epsilon$-stationary point requires $\Omega(\epsilon^{-6})$ stochastic BG-0 oracle queries for smooth functions and $\Omega(\epsilon^{-4})$ queries under mean-square smoothness. These limits demonstrate an unavoidable degradation from classical bounded-variance complexities, i.e., $\Omega(\epsilon^{-4})$ and $\Omega(\epsilon^{-3})$ for smooth and mean-square smooth cases, respectively. To match these lower bounds, we consider \underline{Pr}oximally \underline{A}nchored \underline{ST}ochastic \underline{A}pproximation (PASTA), a unified algorithmic framework that couples Halpern anchoring with Tikhonov regularization to dynamically mitigate the extra variance explosion term permitted by the BG-0 oracle. We prove that PASTA achieves minimax optimal complexities across numerous non-convex regimes, including standard smooth, mean-square smooth, weakly convex, star-convex, and Polyak-Lojasiewicz functions, entirely under an unbounded domain and unbounded stochastic gradients.
\end{abstract}
\section{Introduction}
We consider the canonical unconstrained stochastic optimization problem
\begin{equation}
\min_{x \in \mathbb{R}^p} f(x) = \mathbb{E}_\xi[\tilde{f}(x,\xi)],
\end{equation}
where $f: \mathbb{R}^p \to \mathbb{R}$ is a generally non-convex stochastic objective function. Because exact global minimization of non-convex objectives is intractable in the worst case, the standard algorithmic goal relaxes to identifying an $\epsilon$-stationary point. 

The theoretical analysis of Stochastic Gradient Descent (SGD) and its variance-reduced variants generally relies on the bounded variance assumption $\mathbb{E}[\|\widetilde{\nabla} f(x) - \nabla f(x)\|^2] \le \sigma^2$ for an unbiased estimator $\widetilde{\nabla} f(x)$. While convenient, a uniformly bounded variance fails to hold even in elementary settings such as unconstrained least squares or deep neural network training, where the gradient variance inherently grows with the distance from the optimal solution. While several relaxations are explored in the literature, the Blum-Gladyshev (BG-0) condition \cite{blum1954approximation,gladyshev1965stochastic} is shown to be the weakest viable variance assumption \cite{alacaoglu2025towards}. To capture this structural growth, BG-0 posits that, for a given reference anchor $x_0$, the variance is permitted to scale quadratically: $\mathbb{E}[\|\widetilde{\nabla} f(x) - \nabla f(x)\|^2] \le B_v^2\|x - x_0\|^2 + b_v^2.$
The extra variance term $B_v^2\|x - x_0\|^2$ introduces an undesirable quadratic term in the analyses. Over unbounded domains, this extra variance term breaks the contractive mechanisms of traditional convergence proofs, which rely on the descent lemma, leading to divergence unless highly restrictive assumptions such as an artificially bounded domain are imposed. Let us formalize a generalization of the BG-0 oracle below, which is leveraged for all of our theoretical results.

\begin{assumption}[BG-0 Stochastic Oracle] \label{assump:oracle}
Let $\mathcal{X} \subset \mathbb{R}^p$ be a closed and convex set and consider some $x_0 \in \mathcal{X}$. The stochastic oracle $\widetilde{\nabla} f(x)$ is unbiased, i.e., $\E[\widetilde{\nabla} f(x)] = \bar{g}$ for some $\bar{g} \in \partial f(x)$. Also, there exist constants $B_v \ge 0$ and $b_v \ge 0$ such that
\begin{equation}
\inf_{\bar{g} \in \partial f(x)}\mathbb{E} \left[  \left\| \widetilde{\nabla} f(x)  - \bar{g} \right\|^2 \right]\le B_v^2\|x - x_0\|^2 + b_v^2 \quad \forall x \in \mathcal{X}.
\end{equation}
For a  differentiable function $f$, $\bar{g}$ reduces to the unique gradient $\nabla f(x)$ such that the variance satisfies $\mathbb{E}\left[\|\widetilde{\nabla} f(x) - \nabla f(x)\|^2\right] \le B_v^2\|x - x_0\|^2 + b_v^2.$
\end{assumption}

In convex settings, recent literature \cite{neu2024dealing,alacaoglu2025towards} has successfully leveraged Halpern anchoring \cite{halpern1967fixed} to stabilize trajectories; by drawing the current state toward a static reference point $x_0$ with an anchoring sequence $\beta_t \in (0, 1)$ alongside the step size $\eta_t > 0$, the proximal anchored update is defined by $x_{t+1} = \beta_t x_0 + (1-\beta_t)x_t - \eta_t \widetilde{\nabla} f(x_t).$
This mechanism, along with the assumption of convexity (which inherently enables controlling the growth of the distance), enables us to deal with the extra variance term $B_v^2\|x - x_0\|^2$ and to recover the classical rates under the bounded variance assumption.

However, in contrast with convex functions, the non-convex regimes generally lack global contractive mechanisms to inherently counteract the expanding noise. As all known optimal convergence rates and information-theoretic lower bounds for non-convex stochastic optimization \cite{arjevani2023lower} rely crucially on the uniform bounded-variance assumption; it has remained open whether the classical rates can even be recovered under the BG-0 condition, and more broadly, what the true information-theoretic limits of non-convex optimization are in this relaxed variance regime.

\subsection{Contributions}
Intuitively, one should expect worse complexity bounds under the BG-0 condition for non-convex minimization. If the variance scales quadratically with  distance, an adversarial landscape can force an algorithm to traverse a vast distance from the initialization before encountering the stationary region. As the optimizer traverses this distance, the noise magnitude may explode as a function of $\epsilon^{-1}$, exponentially degrading the signal-to-noise ratio, and as a consequence, the attainable oracle complexity. Somewhat interestingly, we formally leverage this intuition to prove that the celebrated bounded-variance lower bounds of \cite{arjevani2023lower} do not hold under BG-0; the  limits of performance are  worse.

To match these new  limits, we consider a unified algorithmic framework that enables the study of various non-convex regimes without bounded domain and stochastic gradient assumptions. To circumvent the incurred bias of traditional Halpern tracking, which may prevent convergence under non-convexity, we enforce a  coupling between the anchoring parameter and the step size, i.e. $\beta_t = \lambda \eta_t$ for some $\lambda >0$. Under this parameter coupling, the progression of the proximal anchored update becomes algebraically equivalent to a standard step on the strongly convex Tikhonov-regularized surrogate
$F_\lambda(x) = f(x) + \frac{\lambda}{2}\|x - x_0\|^2,$
for a sufficiently large $\lambda.$
We introduce this Tikhonov-regularized surrogate not merely as a stabilization heuristic, but a principled (and to our understanding, necessary) approach to optimize diverse non-convex functions, especially \emph{weakly convex} ones, a rich class that, despite its potentially misleading name, generalizes standard smooth non-convex functions. This specific coupling absorbs the extra variance in Assumption \ref{assump:oracle} by shifting the optimization target to the unique global minimizer of the regularized subproblem. Because targeting $F_\lambda(x)$ induces a stationary bias relative to the true objective $f$, we embed optimizing the Tikhonov-regularized surrogate as an inner loop within an epoch-based framework, sequentially updating the anchor $x_0$ to mitigate the regularization bias. We call this unified approach \underline{P}roximally \underline{A}nchored \underline{ST}ochastic \underline{A}pproximation (PASTA), which is summarized as Algorithm \ref{alg:pasta}. 
From the perspective of fixed-point methods, PASTA acts as the interpolation between two classical methodologies: configuring the algorithm for a single epoch (i.e. $S = 1$) recovers the anchored Halpern iteration \cite{halpern1967fixed} (with our specific coupling), while limiting the inner-loop to a single step (i.e. $K = 1$) reduces the dynamics to the Krasnoselskii-Mann iteration \cite{krasnosel1955two,mann1953mean}.

With this motivation and discussion, we now summarize our key theoretical contributions:

\begin{itemize}[leftmargin=*]
    \item \textbf{Information-Theoretic Lower Bounds:} We establish the oracle complexity lower bounds for finding $\epsilon$-stationary points of smooth functions under the BG-0 condition over unbounded domains. By leveraging the exploding variance in remote landscape regions, we prove that any randomized algorithm requires $\Omega(\epsilon^{-6})$ stochastic gradient evaluations, a degradation from the classical $\Omega(\epsilon^{-4})$ bounded-variance result. Furthermore, under the additional assumption of mean-square smoothness, we establish a tight $\Omega(\epsilon^{-4})$ lower bound, again a degradation from the classical $\Omega(\epsilon^{-3})$ limit.
    
    \item \textbf{Minimax Optimal Upper Bounds:} We prove that PASTA successfully matches these new, degraded lower bounds without requiring bounded domains and bounded stochastic gradient assumptions. We show that specific instantiations of PASTA achieve the optimal $\mathcal{O}(\epsilon^{-6})$ complexity for smooth functions. Moreover, by integrating variance reduction mechanisms, PASTA recovers the optimal $\mathcal{O}(\epsilon^{-4})$ complexity under mean-square smoothness.
    
    \item \textbf{Geometric Universality:} We extend the convergence analysis of PASTA across numerous non-convex regimes. Beyond standard smoothness, we provide optimal convergence guarantees under weak convexity \cite{davis2019stochastic}, star-convexity \cite{nesterov2006cubic}, and the Polyak-Lojasiewicz condition \cite{polyak1963gradient}, demonstrating the universal effectiveness of PASTA against unbounded variance in diverse non-convex stochastic optimization setups.
\end{itemize}

\noindent\textbf{Notation:} We consider a probability space $(\Omega, \mathcal{F}, \mathbb{P})$ and denote the expectation with respect to the underlying probability measure by $\mathbb{E}[\cdot]$. The conditional expectation with respect to the filtration $\mathcal{F}_{t-1}$, representing the history up to the start of iteration $t$, is denoted by $\mathbb{E}_t[\cdot] := \mathbb{E}[\cdot \mid \mathcal{F}_{t-1}]$.
\section{Related Work}\label{sec:related_work}
\noindent\textbf{Variance assumptions.} Classical analyses of SGD  assume uniformly bounded stochastic subgradients \cite{nemirovski2009robust, moulines2011non, rakhlin2012making, shamir2013stochastic}, or globally bounded variance \cite{bottou2018optimization, jain2018parallelizing}. 
To mitigate this, several relaxations have been proposed. The Expected Smoothness or ABC condition \cite{gorbunov2020unified, khaled2022better, khaled2023unified, ilandarideva2023accelerated, grimmer2019convergence} bounds the variance by a constant plus a term proportional to the gradient norm and suboptimality. While less restrictive, \cite{alacaoglu2025towards} demonstrates that the BG-0 condition is weaker and more universally applicable. The origins of the BG-0 condition trace back to the foundational asymptotic analyses of \cite{robbins1951stochastic, blum1954approximation, gladyshev1965stochastic}, and it has seen a recent resurgence in both asymptotic \cite{wang2016stochastic, cui2021analysis} and non-asymptotic contexts \cite{asi2019stochastic, jacobsen2023unconstrained, telgarsky2022stochastic, domke2023provable, neu2024dealing}. Other \emph{orthogonal, complementary relaxations} include heavy-tailed noise settings \cite{gorbunov2023high, nguyen2023improved, gurbuzbalaban2021heavy}, which we do not consider here.

\noindent\textbf{Information-theoretic lower bounds.}
In the noiseless convex setting, \cite{nemirovski1983problem} and \cite{nesterov2004introductory} established the $\Omega(\sqrt{1/\epsilon})$ bounds, while \cite{carmon2019lower_i} proved the $\Omega(\epsilon^{-2})$ bound for smooth non-convex functions. For stochastic convex optimization, a long line of literature has mapped the dimension-dependent and independent boundaries via statistical estimation reductions \cite{raginsky2011information, agarwal2012information, woodworth2016tight, foster2019complexity}.

In the stochastic non-convex regime \cite{arjevani2023lower} proved that under the strictly bounded variance assumption, any randomized algorithm requires $\Omega(\epsilon^{-4})$ queries to find an $\epsilon$-stationary point. Furthermore, under the addition of mean-square smoothness, \cite{arjevani2023lower} establish an accelerated $\Omega(\epsilon^{-3})$ lower bound. Earlier algorithm-specific bounds include \cite{drori2019complexity}, while finite-sum lower bounds were explored in \cite{arjevani2016dimension, fang2018spider}. Our lower bounds conceptually depart from \cite{arjevani2023lower} by leveraging the unbounded nature of the BG-0 oracle. Because the BG-0 variance explodes in remote regions of the landscape, the zero-chain constructions used in \cite{arjevani2023lower} lead to signal-to-noise degradation. By formally embedding a coercive traversal manifold into the zero-chain, we prove that the bounds of \cite{arjevani2023lower} degrade under the BG-0 condition, establishing strictly worse $\Omega(\epsilon^{-6})$ and $\Omega(\epsilon^{-4})$  limits.

\noindent\textbf{Stochastic non-convex and weakly convex minimization.} The study of SGD for non-convex functions gained modern prominence with the work of \cite{ghadimi2013stochastic}, who established the optimal $\mathcal{O}(\epsilon^{-4})$ oracle complexity for finding $\epsilon$-stationary points of smooth functions. Subsequent research heavily explored the capacity of SGD and its perturbed variants to escape saddle points and find local minima \cite{ge2015escaping, jin2015how, lee2016gradient, duchi2018stochastic}. 

Beyond standard smoothness, \emph{weakly convex} stochastic optimization, which encompasses a broader class of non-smooth, non-convex problems prevalent in robust statistics and deep learning, has drawn attention. The seminal work of \cite{davis2019stochastic} developed the stochastic model-based framework, establishing $\mathcal{O}(\epsilon^{-4})$ rates for the stochastic proximal point and prox-linear methods by tracking the stationarity of the Moreau envelope, a natural convergence criterion for the optimization of weakly convex functions. This paved the way for extensive studies into proximal, compositional, constrained, and distributed non-convex methods \cite{duchi2018stochastic, drusvyatskiy2019efficiency, mai2020convergence,huang2023oracle,liu2025single,upadhyay2026fedsgm}. However, the theoretical guarantees in these works implicitly or explicitly rely on globally-bounded variance or bounded domain assumptions, which we relax in this paper.

\noindent\textbf{Variance reduction and high-order methods.}
Large effort has been devoted to surpassing the standard $\mathcal{O}(\epsilon^{-4})$ SGD rate using variance reduction or higher-order information. In the variance reduction category, algorithms based on SVRG, SAGA, SPIDER, SNVRG, and PAGE leverage mean-square smoothness to attain the optimal $\mathcal{O}(\epsilon^{-3})$ rate in various settings including centralized and distributed optimization and fixed-point calculation \cite{allen2016variance, reddi2016stochastic, lei2017non, fang2018spider, zhou2020stochastic, wang2018spiderboost, cutkosky2019momentum, nguyen2017sarah, li2021page,chen2021communication,das2022faster,hashemi2024unified}. Alternatively, high-order methods utilize Lipschitz continuous Hessians to achieve rates of $\mathcal{O}(\epsilon^{-3.5})$ or better via cubic regularization or perturbed acceleration \cite{xu2018first, allen2018neon2, tripuraneni2018stochastic, fang2019sharp}. Crucially, \emph{all} of these accelerated rates depend on the uniform bounded-variance assumption. Our work extends variance reduction mechanics to the BG-0 regime for stochastic non-convex optimization under mean-square smoothness, proving that while acceleration is possible (from $\mathcal{O}(\epsilon^{-6})$ to $\mathcal{O}(\epsilon^{-4})$), the unbounded noise inherently prevents attaining the $\mathcal{O}(\epsilon^{-3})$ rates under bounded variance and mean-square smoothness.

\noindent\textbf{Halpern anchoring and fixed-point dynamics.}
As discussed, PASTA is connected to fixed-point methods. Finding a stationary point maps naturally to finding a fixed point of a non-expansive mapping. Classical approaches include the Krasnoselskii-Mann iteration \cite{krasnosel1955two,mann1953mean}, which lacks strong convergence guarantees, and the Halpern iteration \cite{halpern1967fixed}, which anchors the sequence to a fixed initial point to force strong convergence. 
Recently, Halpern anchoring has been leveraged in convex and min-max optimization to achieve accelerated rates  \cite{diakonikolas2020halpern, yoon2021accelerated}. \cite{neu2024dealing} and \cite{alacaoglu2025towards} demonstrated that Halpern-like updates can counteract the divergence of BG-0 noise in convex and bilinear min-max settings. However, traditional Halpern tracking introduces logarithmic convergence penalties due to uncoupled parameter decay. PASTA, on the other hand, couples the step size and Halpern anchoring parameter via $\beta_t = \lambda \eta_t$. This seamlessly transforms the geometric anchoring into an exact stochastic proximal step on a Tikhonov-regularized surrogate, providing the requisite strong convexity to mitigate the extra variance term in general non-convex and weakly convex landscapes without logarithmic decay artifacts.
\section{Non-Convex Lower Bounds with BG-0 Oracle}
\label{sec:lower_bounds}
In this section, we establish oracle complexity lower bounds for finding $\epsilon$-stationary points of smooth non-convex functions under the BG-0 variance condition over unbounded domains. For smooth functions, $\E[\|\nabla f(x)\|^2] \leq \epsilon^2$ serves as a natural convergence criterion. Classical lower bounds for non-convex stochastic optimization rely on the bounded variance assumption, yielding an optimal rate of $\Omega(\epsilon^{-4})$ for generally smooth functions and $\Omega(\epsilon^{-3})$ when mean-square smoothness holds \cite{arjevani2023lower}. We demonstrate that allowing the variance to grow quadratically with the distance inherently degrades the information-theoretic limits of the optimization landscape. Let us first state smoothness formally, and then we state our lower bound under Assumption \ref{assump:oracle}.

\begin{assumption}[Smoothness] \label{assump:smooth}
A function $f: \mathbb{R}^p \to \mathbb{R}$ is $L$-smooth if for all $x, y \in \mathbb{R}^p$ it holds that
$\|\nabla f(x) - \nabla f(y)\| \le L\|x - y\|.$
\end{assumption}

\begin{theorem}[Lower Bound under Smoothness] \label{thm:bg0_lower_bound}
For any $L > 0$, $\Delta > 0$, $B_v \ge 0$, $b_v \ge 0$, and $\epsilon \le \sqrt{\frac{L\Delta}{1536\ell_1}}$ for some numerical constant $\ell_1>0$, any randomized algorithm requires $\Omega\bigl( \frac{B_v^2 L \Delta^3}{\epsilon^6} + \frac{b_v^2 L \Delta}{\epsilon^4} \bigr)$ stochastic oracle queries in expectation to find a point $x$ satisfying $\mathbb{E}[\|\nabla F(x)\|] \le \epsilon$ for a function $F$ that satisfies $F(x_0) - \inf_x F(x) \le \Delta$, and Assumptions \ref{assump:oracle} and \ref{assump:smooth}.
\end{theorem}

The proof is in Appendix \ref{sec:bg0_lower_bound}. The emergence of the $\Omega(\epsilon^{-6})$ complexity represents a degradation compared to the classical $\Omega(\epsilon^{-4})$ rate. To understand the source of this lower bound, it is instructive to contrast our hard instance construction with the seminal bounded-variance framework established by \cite{arjevani2023lower}. Under a uniform variance bound $\sigma^2$, the standard lower bound works by constructing a high-dimensional probabilistic zero-chain $F:\mathbb{R}^T \rightarrow \mathbb{R}$ (see \eqref{eq:f_unscaled} for its definition). In this chain, the true gradient signal is masked by noise, requiring $\Omega(\sigma^2/\epsilon^2)$ queries in expectation to discover each subsequent relevant coordinate. To satisfy the stationarity condition, an algorithm must activate $\Theta(\Delta L / \epsilon^2)$ such coordinates, thereby leading to the $\Omega(\Delta L \sigma^2 / \epsilon^4)$ bound.

Our construction, on the other hand, leverages the dependence of the BG-0 oracle on distances against the optimization algorithm. First, we consider a suitable scaling of the zero-chain $F_{\text{scaled}}:\mathbb{R}^T \rightarrow \mathbb{R}$, critically requiring a shift not used by \cite{arjevani2023lower} (see \eqref{eq:f_scaled} for its definition). Then, we isolate the information bottleneck by coupling a one-dimensional travel function $f_0(u)$ with the zero-chain subspace (see \eqref{eq:travel} for its definition). The composite objective $F(u, y) = f_0(u) + \phi(u) F_{\text{scaled}}(y)$ requires that the algorithm traverse  $D = \Theta(\Delta/\epsilon)$ along the travel coordinate $u$ to drop the gradient norm below the threshold $\epsilon$. We use a smooth activation function $\phi(u)$ to ensure that the gradient subspace corresponding to the high-dimensional zero-chain $F_{\text{scaled}}(y)$ remains inactive until this mandatory traversal is completed (see \eqref{eq:activation} for its definition).

By the time the algorithm reaches the active region of the zero-chain, its distance from the initialization $x_0$  satisfies $\|x - x_0\|^2 \ge \Theta(\Delta^2/\epsilon^2)$. The BG-0 condition thus dictates that the variance at this remote location is $\tilde{\sigma}^2 = \Theta(B_v^2 \Delta^2 / \epsilon^2 + b_v^2)$. By substituting this big variance budget into the standard zero-chain mechanics of \cite{arjevani2023lower}, the probability of the oracle revealing the true gradient direction shrinks from $\Theta(\epsilon^2)$ to $\Theta(\epsilon^4)$. The algorithm is thus forced to optimize the non-convex structure exactly where the signal-to-noise ratio is worst, yielding the $\Omega(\epsilon^{-6})$ limit.

We now focus on the case of having Mean-Square Smoothness (MSS), which enables improving the lower bound. The condition is defined next, followed by our second lower bound.
\begin{assumption}[Mean-Square Smoothness] \label{assump:ms-smooth}
A function $f: \mathbb{R}^p \to \mathbb{R}$ is $\bar{L}$-smooth in the mean square sense if for all $x, y \in \mathbb{R}^p$ it holds that $\mathbb{E} [\| \widetilde{\nabla} f(x)  - \widetilde{\nabla} f(y) \|^2] \le \bar{L}^2\|x - y\|^2.$
\end{assumption}

\begin{theorem}[Lower Bound under Mean-Square Smoothness] \label{thm:bg0_mss_lower_bound}
For any $\bar{L} > 0$, $\Delta > 0$, $b_v \ge 0$, $0\le B_v \le \frac{\varsigma \bar{L}}{96 \bar{\ell}_1}$ for some numerical constants $\bar{\ell}_1>0$, $\varsigma>0$, and with
\begin{equation}
0<\epsilon \le \frac{\sqrt{-2 b_v^2 + \sqrt{4 b_v^4 + \varsigma^2 \left( \frac{\varsigma^2 \bar{L}^2 \Delta^2}{9216 \bar{\ell}_1^2} - B_v^2 \Delta^2 \right)}}}{4 \varsigma},
\end{equation}
any randomized algorithm requires $\Omega\bigl( \frac{B_v \bar{L} \Delta^2}{\epsilon^4} + \frac{b_v \bar{L} \Delta}{\epsilon^3} + \frac{B_v^2 \Delta^2}{\epsilon^4} + \frac{b_v^2}{\epsilon^2} \bigr)$ stochastic oracle queries in expectation to find a point $x$ satisfying $\mathbb{E}[\|\nabla F(x)\|] \le \epsilon$ for a function $F$ that satisfies $F(x_0) - \inf_x F(x) \le \Delta$, and Assumptions \ref{assump:oracle} and \ref{assump:ms-smooth}.
\end{theorem}

The proof is in Appendix \ref{sec:bg0_mss_lower_bound}. Theorem \ref{thm:bg0_mss_lower_bound} formally establishes that while MSS helps improve the complexity, it cannot break the $\Omega(\epsilon^{-4})$ barrier under the BG-0 condition, deviating sharply from the seminal $\mathcal{O}(\epsilon^{-3})$ bounds attainable under uniform variance and MSS. 
The bottleneck due to the extra variance term allowed by BG-0 is manifested in two terms. The term $\Omega(B_v \bar{L} \Delta^2 / \epsilon^4)$ characterizes the sequential optimization complexity. To maintain the global MSS condition characterized by $\bar{L}$ despite the massive extra variance at distance $D$, the deterministic smoothness $L$ of the underlying zero-chain must be controlled. Doing so couples the chain length to the variance limit, forcing more sequential queries to overcome the progressive information masking. 
Conversely, the term $\Omega(B_v^2 \Delta^2 / \epsilon^4)$ represents a pure statistical estimation floor. This bound operates independently of the zero-chain. Any algorithm must evaluate gradients at the traversal boundary imposed by our construction to verify the $\epsilon$-stationarity condition. Because the oracle introduces noise with variance $\Theta(B_v^2 \Delta^2 / \epsilon^2)$, even near an $\epsilon$-stationarity solution, standard information-theoretic arguments for mean estimation furnish that simply computing a single valid gradient vector to precision $\epsilon$ requires $\Omega(\epsilon^{-4})$ independent stochastic samples. Consequently, in highly stochastic regimes where $B_v \gg \bar{L}$, the inherent variance permitted by the BG-0 oracle exceeds the sharpness of the objective function (manifested by the MSS constant $\bar{L}$), reducing the global non-convex optimization problem to a seminal statistical estimation one.
\section{Let's Cook PASTA!}\label{sec:pasta}
We now introduce PASTA, outlined in Algorithm \ref{alg:pasta}. The primary algorithmic challenge imposed by the BG-0 oracle condition is the unbounded growth of the stochastic gradient variance as the sequence of iterates drifts away from the initialization. Traditional unanchored stochastic gradient methods lack the necessary contractive mechanism to suppress this extra variance, leading to divergence over unbounded domains. PASTA mitigates this issue by embedding a Halpern-style anchoring mechanism as explored by \cite{neu2024dealing,alacaoglu2025towards}, but crucially, within a multi-epoch proximal point architecture.

Let us first talk about the update step used in the inner loop. Under MSS, classical methods such as SGD and SGD with momentum are sub-optimal, so they will also be sub-optimal under the BG-0 oracle. Optimal methods rely on variance reduction. Two well-known methods include PAGE and STORM \cite{li2021page,cutkosky2019momentum}. Focusing on the former without loss of generality in this paper, to ensure we meet the lower bound established by Theorem \ref{thm:bg0_mss_lower_bound} under MSS, the recursive estimator replacing the vanilla estimator of SGD takes the form
\begin{equation}\label{eq:page}
g_t = \begin{cases} 
\frac{1}{N_t} \sum_{i=1}^{N_t} \widetilde{\nabla} f(x_t, \xi_t^{(i)}) & \text{with probability } p_t \\
g_{t-1} + \frac{1}{b_t} \sum_{i=1}^{b_t} \left( \widetilde{\nabla} f(x_t, \zeta_t^{(i)}) - \widetilde{\nabla} f(x_{t-1}, \zeta_t^{(i)}) \right) & \text{with probability } 1 - p_t 
\end{cases}
\end{equation}
where $N_t$ is a large batch size and $b_t$ is a small secondary batch size. Note that when $p_t = 1$, the update reduces to that of simple SGD.

As mentioned, PASTA operates over $S$ consecutive epochs, each consisting of $K$ inner stochastic gradient steps. At the beginning of each epoch $s$, the algorithm establishes a fixed local anchor $x_{s-1}$. During the inner iterations, PASTA forms an update vector according to \eqref{eq:page} 
and executes the coupled update $x_{t+1} = \beta_t x_{s-1} + (1-\beta_t)x_t - \eta_t g_t.$ Note that, when $p_t = 1$, the update vector reduces to a simple mini-batch stochastic gradient $g_t$ of size $N_t$. 
In that case, this update step interpolates between the static anchor $x_{s-1}$ and standard SGD.

A key feature of PASTA, particularly for weakly convex analysis, is the coupling of the anchoring parameter $\beta_t$ and the step size $\eta_t$. By enforcing the relation $\beta_t = \lambda \eta_t$ for some regularization parameter $\lambda > 0$, the update naturally reformulates as an attempt to optimize the Tikhonov-regularized surrogate objective $F_\lambda(x) = f(x) + \frac{\lambda}{2}\|x - x_{s-1}\|^2.$
This regularization induces $\lambda$-strong convexity into the local optimization landscape. The induced strong convexity dynamically counteracts the extra variance term $B_v^2\|x - x_0\|^2$ in the BG-0 assumption, thereby bounding the maximum experienced variance.
However, optimizing the surrogate $F_\lambda(x)$ inherently introduces a stationary bias relative to the true objective $f(x)$. PASTA resolves this by systematically resetting the anchor to the final iterate of the previous epoch, $x_s = x_K$. This outer-loop mechanism mirrors the classical inexact proximal point algorithm, progressively removing the regularization bias and driving the sequence toward a true $\epsilon$-stationary point of the original non-convex objective. By configuring the epoch length $K$, the batch size $N_t$, and other parameters, PASTA can recover minimax optimal complexities across a wide spectrum of non-convex regimes, as we show in the next section.
\begin{algorithm}[t]
\caption{Proximally Anchored STochastic Approximation (PASTA)}
\label{alg:pasta}
\begin{algorithmic}[1]
\STATE \textbf{Parameters:} Initial point $x_0$, epoch length $K$, total epochs $S$, step size $\eta_t > 0$, anchoring parameter $\beta_t$, reset parameter $p_t$.
\FOR{$s = 1, 2, \dots, S$}
\FOR{$t = 0, 1, \dots, K-1$}
\STATE Compute $g_t$ according to \eqref{eq:page} by sampling the BG-0 oracle $N_t$ or $b_t$ times depending on $u \sim\mathrm{Ber}(p_t)$
\STATE Do the coupled proximal anchored update $x_{t+1} = \beta_t x_{s-1} + (1-\beta_t)x_t - \eta_t g_t.$
\ENDFOR
\STATE Set the output of epoch $s$ to the last iterate $x_s = x_K.$
\ENDFOR
\end{algorithmic}
\end{algorithm}
\section{Non-Convex Upper Bounds with BG-0 Oracle}\label{sec:upper_bounds}
Having established the lower bounds under the BG-0 condition in Section \ref{sec:lower_bounds}, we now demonstrate that PASTA achieves these limits with the right instantiation of the parameters. Beyond standard smoothness and MSS, we also consider other classes of non-convex functions.

\subsection{Optimal Complexity under Smoothness}\label{sec:upper_bounds_smooth}
Under standard $L$-smoothness, our lower bound is $\Omega(\epsilon^{-6})$. To achieve this limit, we note that we can directly control the BG-0 variance by varying the sampling budget. In particular, we use a single epoch ($S=1$) and eliminate the anchor ($\beta_t=0$), such that with $p_t=1$, PASTA simplifies to standard unanchored SGD. To prevent the extra variance term allowed by the BG-0 oracle from causing convergence issues, we scale the batch size dynamically at each step, proportional to the squared distance from the initialization. This ensures the variance of the mini-batch estimator remains bounded by a global constant, achieving the $\mathcal{O}(\epsilon^{-6})$ limit through sampling adaptation, leading to the following result.

\begin{theorem}[Convergence of PASTA via Dynamic Batching - Smooth Non-convex Case]\label{thm:pasta_dynamic}
Let Assumptions \ref{assump:oracle} and \ref{assump:smooth} hold. Let the sequence $\{x_t\}$ be generated by PASTA with $S=1$, $\beta_t = 0$, and $p_t=1$. Given a constant target variance $\sigma^2 > 0$, set batch size $N_t$, the step size $\eta$, and total iterations $K$ according to
\begin{equation}
N_t = \max\left\{1,\left\lceil \frac{B_v^2 \|x_t - x_0\|^2 + b_v^2}{\sigma^2} \right\rceil\right\}, \quad     \eta = \min\left\{\frac{1}{L}, \frac{\epsilon^2}{2 L \sigma^2}\right\}, \quad K = \left\lceil \frac{4 \Delta}{\eta \epsilon^2} \right\rceil,
\end{equation}
where $\Delta \ge f(x_0) - \inf_x f(x)$. Then, $\frac{1}{K} \sum_{t=0}^{K-1} \mathbb{E}[\|\nabla f(x_t)\|^2] \le \epsilon^2$ and the expected total stochastic oracle complexity evaluates to $\mathcal{O}(\epsilon^{-6})$.
\end{theorem}

The proof is in Appendix \ref{sec:pasta_dynamic}, where the exact expected oracle complexity can be found in \eqref{eq:exact_bound_dynamic}. Theorem \ref{thm:pasta_dynamic} matches the established lower bound in Theorem \ref{thm:bg0_lower_bound}.
\subsection{Optimal Complexity under Convexity}\label{sec:upper_bounds_cvx}
While our focus is on non-convex minimization, a natural question is whether the PASTA instantiation we explored above (i.e., the good old unanchored SGD with dynamic batching) can recover the rates established in \cite{alacaoglu2025towards} for convex functions. There, \cite{alacaoglu2025towards} demonstrates that SGD with Halpern anchoring attains $\mathcal{O}(\epsilon^{-2})$ complexity under the BG-0 oracle assumption. Our next theorem below provides an affirmative answer to the preceding question. 
\begin{theorem}[Convergence of PASTA via Dynamic Batching - Smooth Convex Case]\label{thm:pasta_dynamic2}
Let Assumptions \ref{assump:oracle} and \ref{assump:smooth} hold. Furthermore, let $f$ be convex (See Assumption \ref{assump:weak_convexity} with $\rho=0$). Let the sequence $\{x_t\}$ be generated by PASTA with $S=1$, $\beta_t = 0$, and $p_t=1$. Given a constant target variance $\sigma^2 > 0$, set batch size $N_t$, the step size $\eta$, and total iterations $K$ according to
\begin{equation}
N_t = \max\left\{1,\left\lceil \frac{B_v^2 \|x_t - x_0\|^2 + b_v^2}{\sigma^2} \right\rceil\right\}, \quad \eta = \min\left\{\frac{1}{2L},\frac{\epsilon}{2\sigma^2}\right\}, \quad K = \left\lceil \frac{2\|x_0 - x^*\|^2}{\eta \epsilon} \right\rceil.
\end{equation}
Then, for the ergodic average $\bar{x}_K = \frac{1}{K}\sum_{t=0}^{K-1} x_t$, $\mathbb{E}[f(\bar{x}_K) - f(x^*)] \le \epsilon$ and the expected total stochastic oracle complexity evaluates to $\mathcal{O}(\epsilon^{-2})$.
\end{theorem}

The proof is in Appendix \ref{sec:pasta_dynamic_cvx}, where the exact expected oracle complexity can be found in \eqref{eq:n_total} and below. Using a similar technique, we can state the following theorem for convex functions that are not smooth, the proof of which is in Appendix \ref{sec:pasta_lip_cvx}, where the exact expected oracle complexity can be found in  \eqref{eq:exact_bound_lip_cvx}.
\begin{theorem}[Convergence of PASTA under Lipschitz Continuity and Convexity]\label{thm:pasta_lipschitz_convex}
Let Assumption \ref{assump:oracle} hold and assume $f$ is convex (see Assumption \ref{assump:weak_convexity} with $\rho=0$ and $G$-Lipschitz (see Assumption \ref{assump:lipschitz}). Let the sequence $\{x_t\}$ be generated by PASTA with $S=1$, $\beta_t = 0$, and $p_t=1$. Given a constant target variance $\sigma^2 > 0$, set the batch size $N_t$, the step size $\eta$, and total iterations $K$ according to
\begin{equation}
N_t = \max\left\{1, \left\lceil \frac{B_v^2 \|x_t - x_0\|^2 + b_v^2}{\sigma^2} \right\rceil\right\}, \quad \eta = \frac{\epsilon}{G^2 + \sigma^2}, \quad K = \left\lceil \frac{(G^2 + \sigma^2)\|x_0 - x^*\|^2}{\epsilon^2} \right\rceil.
\end{equation}
Then, for the ergodic average $\bar{x}_K = \frac{1}{K}\sum_{t=0}^{K-1} x_t$, $\mathbb{E}[f(\bar{x}_K) - f(x^*)] \le \epsilon$ and the expected total stochastic oracle complexity evaluates to $\mathcal{O}(\epsilon^{-2})$.
\end{theorem}
Note that one could avoid using the Lipschitzness assumption if, following \cite{neu2024dealing,alacaoglu2025towards} the variance condition in Assumption \ref{assump:oracle} is strengthened to $\mathbb{E} [\| \widetilde{\nabla} f(x) \|^2]\le B_v^2\|x - x_0\|^2 + b_v^2$.
\subsection{Geometric-Statistical Uncertainty Principle}
Our results in Section \ref{sec:upper_bounds_cvx}, along with those of \cite{alacaoglu2025towards}, demonstrate two distinct instantiations of PASTA to obtain optimal complexity under the BG-0 oracle and convexity. A natural question is whether these two instances are related, as one is anchored and uses a single oracle call per iteration, whereas the other is unanchored and uses dynamic batching. In what follows, we show that these two distinct instantiations are conjugate and satisfy an uncertainty principle.

The derivations in Appendices \ref{sec:pasta_dynamic_cvx} and \ref{sec:pasta_lip_cvx} highlight that the key mechanism required to guarantee convergence under the BG-0 oracle is a bounded expected drift of the form $\mathbb{E}_t[\|x_{t+1} - x^*\|^2] \le \bigl(1 - \frac{\beta_t}{2}\bigr)\|x_t - x^*\|^2 + \mathcal{E}_t$ for some residual $\mathcal{E}_t$. 

Under $\sigma^2$-bounded variance, the residual $\mathcal{E}_t$ is bounded by a global constant controlled by $\eta$ and the variance $\sigma^2$, enabling convergence.
However, the BG-0 oracle permits the variance to scale quadratically with the trajectory distance. If the residual error $\mathcal{E}_t$ retains the unbounded term $\|x_t - x_0\|^2$, convergence cannot be established. Therefore, in the following result (proof in Appendix \ref{sec:u_p}), we establish conditions to obtain a bounded residual under the BG-0 oracle.

\begin{theorem}[Geometric-Statistical Uncertainty Principle] \label{thm:uncertainty_dynamic}
Let Assumptions \ref{assump:oracle} and \ref{assump:smooth} hold. Furthermore, let $f$ be convex (See Assumption \ref{assump:weak_convexity} and set $\rho=0$). Consider the sequence $\{x_t\}$ generated by $x_{t+1} = \beta_t x_0+(1-\beta_t)x_t - \eta g_t$, where $g_t$ evaluates $N_t \ge 1$ independent stochastic queries (i.e., a PASTA instance with $S=1$ and $p_t=1$). Then, to ensure
\begin{equation}
   \mathbb{E}_t\|x_{t+1} - x^*\|^2 \le \left(1 - \frac{\beta_t}{2}\right)\|x_t - x^*\|^2 +\mathcal{E}_t,
\end{equation}
where the residual $\mathcal{E}_t =  \beta_t\|x_0 - x^*\|^2 + \frac{\eta^2 b_v^2}{N_t}$ is independent of the unbounded term $\|x_t - x_0\|^2$, the following uncertainty principle must hold
\begin{equation}
\beta_t(1 - 2\beta_t)N_t \ge \eta^2 B_v^2,\qquad \eta \le \frac{\sqrt{\beta_t}}{2 L}.
\end{equation}
\end{theorem}
The relation $\beta_t(1 - 2\beta_t)N_t \ge \eta^2 B_v^2$ is to be interpreted as an uncertainty principle. The left side represents PASTA's stabilization capacity, determined by the geometric parameter $\beta_t \ll 1$ and the statistical parameter $N_t$. The right side quantifies the normalized variance growth over a single step. The lower bound indicates that eliminating geometric anchoring, i.e., $\beta_t \to 0$, necessitates an unbounded expansion in statistical sampling, i.e., $N_t \to \infty$, to counteract the extra noise permitted by the BG-0 oracle (corresponding to the PASTA instantiation explored here). Conversely, using $N_t = 1$ requires setting $\beta_t \sim \eta^2 B_v^2$ to counteract the extra noise permitted by the BG-0 oracle (corresponding to the PASTA instantiation explored in \cite{alacaoglu2025towards}). Therefore, the uncertainty principle rules out the simplest instance, i.e., SGD with $\beta_t \to 0$ and $N_t = 1$, simultaneously. Interestingly, the scaling $\beta_t \sim \eta^2$ when $\eta \sim 1/\sqrt{t}$ matches the classical anchoring of Halpern iteration \cite{halpern1967fixed}.
\subsection{Optimal Complexity under Mean-Square Smoothness}\label{sec:upper_bounds_ms-smooth}
When the stochastic oracle additionally satisfies the MSS condition, the gradients exhibit sufficient correlation to employ recursive variance reduction. Under standard bounded variance, techniques such as PAGE drop the oracle complexity from $\mathcal{O}(\epsilon^{-4})$ to $\mathcal{O}(\epsilon^{-3})$. However, as dictated by our lower bound in Theorem \ref{thm:bg0_mss_lower_bound}, the best we can expect to do with variance reduction under the BG-0 oracle is $\mathcal{O}(\epsilon^{-4})$. To achieve this limit, we instantiate PASTA with $p_t <1$. This choice is essential, as demonstrated by the following key result from \cite{li2021page}, to ensure variance reduction is realized.
\begin{lemma}[PAGE Estimation Error \cite{li2021page}] \label{lem:page_variance}
Let the estimator $g_t$ be generated by \eqref{eq:page}, and define the estimation error $V_t = \|g_t - \nabla f(x_t)\|^2$. Under Assumption \ref{assump:ms-smooth}, it holds that
\begin{equation}
\mathbb{E}_t[V_t] \le p_t \mathbb{E}_t \left[ \left\| \frac{1}{N_t} \sum_{i=1}^{N_t} \widetilde{\nabla} f(x_t, \xi_t^{(i)}) - \nabla f(x_t) \right\|^2 \right] + (1 - p_t) \left( V_{t-1} + \frac{\bar{L}^2}{b_t} \|x_t - x_{t-1}\|^2 \right).
\end{equation}
\end{lemma}
The primary challenge under the BG-0 oracle is that recursive estimators bound variance using the distance between consecutive iterates, while the BG-0 noise scales with the absolute distance to the origin. We thus ensure PASTA utilizes dynamic large batches, an idea akin to what was explored in Theorem \ref{thm:pasta_dynamic}, leading to the following result, the proof of which is in Appendix \ref{sec:page_bg0_upper_bound}, where one can find the exact expected oracle complexity in \eqref{eq:exact_bound_mss}. 

\begin{theorem}[Convergence of PASTA under MSS]\label{thm:page_bg0_upper_bound}
Let Assumptions \ref{assump:oracle} and \ref{assump:ms-smooth} hold.  Let the sequence $\{x_t\}$ be generated by PASTA with
$\beta_t = 0$, $S=1$. Set
\begin{equation}
    N_{\text{ref}} = \max\left\{1, \left\lceil \frac{8 B_v^2 \Delta^2}{\epsilon^4} + \frac{2 b_v^2}{\epsilon^2} \right\rceil\right\}, \quad b = \left\lceil \sqrt{N_{\text{ref}}} \right\rceil, \quad p = \frac{b}{N_{\text{ref}}}, \quad \eta = \frac{1}{4\bar{L}},\quad K = \left\lceil \frac{16 \Delta \bar{L}}{\epsilon^2} + \frac{1}{p} \right\rceil,
\end{equation}
where $f(x_0) - \inf_x f(x) \le \Delta$. Set batch size $N_0$ and set batch size $N_t$ for $t \ge 1$ as
\begin{equation}
    N_0 = \max\left\{1, \lceil 2 b_v^2 \epsilon^{-2} \rceil\right\}, \quad N_t = \max \left\{ N_{\text{ref}}, \left\lceil \frac{2(B_v^2 \|x_t - x_0\|^2 + b_v^2)}{\epsilon^2} \right\rceil \right\}.
\end{equation}
Then, $\frac{1}{K} \sum_{t=0}^{K-1} \mathbb{E}[\|\nabla f(x_t)\|^2] \le \epsilon^2$ and the expected total stochastic oracle complexity evaluates to $\mathcal{O}\left(\epsilon^{-4}\right).$
\end{theorem}
\subsection{Optimal Complexity under Global Curvature}\label{sec:pl_star_upper_bounds}
We now shift our focus to optimization landscapes characterized by a degree of global curvature toward the optimal set, specifically those satisfying the Polyak-Lojasiewicz (PL) condition and star-convexity formalized next. We note that given the additional structure, $\E[ f(x)-f(x^\ast)] \leq \epsilon$ serves as our natural convergence criterion. 
\begin{assumption}[Polyak-Lojasiewicz Condition] \label{assump:pl_smooth}
A function $f: \mathbb{R}^p \to \mathbb{R}$ satisfies the Polyak-Lojasiewicz (PL) condition with constant $\mu > 0$, if $\|\nabla f(x)\|^2 \ge 2\mu(f(x) - \min f)$ for all $x$.
\end{assumption} 
\begin{assumption}[Star-Convexity] \label{assump:star}
A function $f: \mathbb{R}^p \to \mathbb{R}$ is  $\mu$-star-convex on $\mathbb{R}^p$ with coefficient $\mu > 0$, if $\langle \nabla f(x), x - x^* \rangle \ge f(x) - f(x^*) + \frac{\mu}{2}\|x - x^*\|^2,$
for some global minimizer $x^*$ and for all $x$.
\end{assumption}
In these regimes, the extra variance allowed by the BG-0 oracle does not degrade the asymptotic convergence rates.
The intuition underlying this recovery relies on the Quadratic Growth (QG) property. Both the PL condition and star-convexity imply QG \cite{karimi2016linear} (with star-convexity also implying PL), which posits $\|x - x^*\|^2 \le \frac{2}{\mu}(f(x) - f(x^*))$. Expanding the variance, we can write $\|x_t - x_0\|^2 \le 2\|x_t - x^*\|^2 + 2\|x_0 - x^*\|^2$. By invoking QG, we map this variance directly back to the functional suboptimality. Consequently, the extra noise injected by the oracle is upper-bounded by a combination of the current optimization error and the initialization error. Because the gradient signal in PL and star-convex landscapes provides a linear contraction, it naturally absorbs the multiplicative variance amplification all by itself, leading to the following results.

\begin{theorem}[Convergence of PASTA under PL Condition]\label{thm:pasta_pl}
Let Assumptions \ref{assump:oracle}, \ref{assump:smooth}, and \ref{assump:pl_smooth}  hold. Let the sequence $\{x_t\}$ be generated by PASTA with $N_t=1$, $S=1$, $p_t =1$, and $\beta_t = 0$.
Set
\begin{equation}
\eta = \min\left\{ \frac{1}{L}, \frac{\mu \epsilon}{3 L b_v^2}, \frac{\mu^2 \epsilon}{12 L B_v^2 \Delta } \right\}, \quad K = \left\lceil \frac{2}{\eta\mu} \left[\log\left(\frac{3\Delta }{\epsilon}\right)\right]_+ \right\rceil, \text{ where } [a]_+ := \max\{a,0\}
\end{equation}
where $\Delta \ge f(x_0) - \inf_x f(x)$. Then, $\mathbb{E}[f(x_K) - f(x^*)] \le \epsilon$, evaluating to a total stochastic oracle complexity of $\mathcal{O}(\epsilon^{-1} \log(\epsilon^{-1}))$.
\end{theorem}
The proof is in Appendix \ref{sec:pasta_pl}, where the exact oracle complexity can be found in \eqref{eq:exact_bound_pl}. 
We now establish an analogous result for star-convexity directly. While star-convexity also implies QG, the descent analysis operates directly on the sequence $\mathbb{E}[\|x_t - x^*\|^2]$ rather than the functional sequence $\mathbb{E}[f(x_t) - f(x^*)]$, the proof of which is in Appendix \ref{sec:pasta_star},  where the exact oracle complexity can be found in \eqref{eq:exact_bound_star}.

\begin{theorem}[Convergence of PASTA under Star-Convexity]\label{thm:pasta_star}
Let Assumptions \ref{assump:oracle}, \ref{assump:smooth}, and \ref{assump:star} hold. Let the sequence $\{x_t\}$ be generated by PASTA with $N_t=1$, $S=1$, $p_t =1$, and $\beta_t = 0$. Set
\begin{equation}
\eta = \min\left\{ \frac{1}{L}, \frac{\mu}{4 B_v^2}, \frac{\mu^2 \epsilon}{12 L B_v^2 \Delta }, \frac{\mu \epsilon}{3 L b_v^2} \right\}, \qquad K = \left\lceil \frac{2}{\eta\mu} \left[\log\left( \frac{3 L \Delta }{\mu \epsilon} \right)\right]_+ \right\rceil,
\end{equation}
where $\Delta \ge f(x_0) - \inf_x f(x)$. Then, $\mathbb{E}[f(x_K) - f(x^*)] \le \epsilon$, evaluating to a total stochastic oracle complexity of $\mathcal{O}(\epsilon^{-1} \log(\epsilon^{-1}))$.
\end{theorem}
\subsection{Optimal Complexity under Weak Convexity}\label{sec:weak_convexity_upper}
We now address the weakly convex regime defined formally next. 
\begin{assumption}[Weak Convexity] \label{assump:weak_convexity}
A function $f: \mathbb{R}^p \to \mathbb{R}$ is  $\rho$-weakly convex for some constant $\rho \ge 0$ if $f(x) + \frac{\rho}{2}\|x\|^2$ is convex.
Note that when $\rho = 0$, the function is convex in the typical sense. Also, any smooth function is $L$-weakly convex.
\end{assumption}
Analysis of weakly convex functions requires imposing Lipschitzness. Note that one could avoid using the Lipschitzness assumption stated below for the ensuing analysis if, following \cite{neu2024dealing,alacaoglu2025towards}, the variance condition in Assumption \ref{assump:oracle} is strengthened to $\mathbb{E} [\| \widetilde{\nabla} f(x) \|^2 ]\le B_v^2\|x - x_0\|^2 + b_v^2$.
\begin{assumption}[Lipschitz Continuity] \label{assump:lipschitz}
A function $f: \mathbb{R}^p \to \mathbb{R}$ is $G$-Lipschitz continuous if $\|g\| \le G$ for all $x$ and $g \in \partial f(x)$.
\end{assumption}

Unlike the PL or star-convex cases, weakly convex functions, which generalize smooth functions, lack global curvature pointing toward the optimal set. Additionally, due to a lack of bounded curvature, the unanchored SGD approach utilized in preceding sections will fail. We thus use PASTA in its epoch-based form with a nonzero $\beta_t = \lambda \eta_t$. As we discussed in Section \ref{sec:pasta}, because $f$ is $\rho$-weakly convex, selecting the coupling parameter $\lambda > \rho$ ensures that the surrogate $F_\lambda(x) = f(x) + \frac{\lambda}{2}\|x - x_0\|^2$ is  $\mu$-strongly convex with coefficient $\mu = \lambda - \rho > 0$, a property akin to the curvature provided by the PL condition or star-convexity. However, minimizing the surrogate $F_\lambda(x)$ only yields the minimizer of the regularized subproblem, resulting in a biased stationary point of the true objective $f$. To remove this bias, PASTA employs an epoch-based structure. By running the inner loop for $K$ iterations and subsequently updating the anchor $x_s = x_t$, the outer loop effectively implements an inexact Stochastic Proximal Point algorithm. 

For weakly convex functions, following \cite{davis2019stochastic}, we employ the Moreau envelope as a continuous surrogate to quantify stationarity. We recall some definitions from \cite{davis2019stochastic}.

\begin{definition}[Proximal Operator]
For a proper, lower semicontinuous function $f: \mathbb{R}^p \to \mathbb{R} \cup \{\infty\}$ and a parameter $\lambda > 0$, the proximal mapping is defined by
\begin{equation}
\mathrm{prox}_{f/\lambda }(x) = \underset{y \in \mathbb{R}^p}{\mathrm{argmin}} \left\{ f(y) + \frac{\lambda}{2} \|y - x\|^2 \right\}.
\end{equation}
\end{definition}

\begin{definition}[Moreau Envelope]
Associated with the proximal operator, the Moreau envelope of $f$ with parameter $\lambda > 0$ is defined as
\begin{equation}
\varphi_{1/\lambda}(x) = \min_{y \in \mathbb{R}^p} \left\{ f(y) + \frac{\lambda}{2} \|y - x\|^2 \right\}.
\end{equation}
\end{definition}

The Moreau envelope acts as a regularized, smoothed approximation of $f$. When $f$ is $\rho$-weakly convex and $\lambda > \rho$, the objective of the proximal mapping is strongly convex, ensuring that $\mathrm{prox}_{f/\lambda}(x)$ is unique and well-defined everywhere. The following lemma connects the gradient of the Moreau envelope to the proximal step.

\begin{lemma} \label{lem:moreau_gradient}
Let $f$ be a $\rho$-weakly convex function and $\lambda > \rho$. Then, $\varphi_{1/\lambda}(x)$ is smooth with constant $L_{\varphi} = \max \bigl\{ \lambda, \frac{\lambda \rho}{\lambda - \rho} \bigr\}$, and its gradient is given by $\nabla \varphi_{1/\lambda}(x) = \lambda \bigl( x - \mathrm{prox}_{f/\lambda}(x) \bigr).$
Furthermore, $\|\nabla \varphi_{1/\lambda}(x)\|$ serves as a valid surrogate for stationarity. If $\|\nabla \varphi_{1/\lambda}(x)\| \le \epsilon$, then there exists a point $\hat{x} = \mathrm{prox}_{f/\lambda}(x)$ such that $\|x - \hat{x}\| \le \lambda^{-1} \epsilon$ and $\mathrm{dist}(0, \partial f(\hat{x})) \le \epsilon$. Finally, it holds that $\varphi_{1/\lambda}(x_0) - \min_x \varphi_{1/\lambda}(x) \le f(x_0) - \inf_x f(x)$ for all $x_0$ and $\lambda>0$.
\end{lemma}
Note that since any smooth function is weakly convex, the lower bound established in Theorem \ref{thm:bg0_lower_bound} is also applicable here, and the best complexity one could hope for is $\mathcal{O}(\epsilon^{-6})$.

Having defined our notation of stationarity, the following theorem formalizes the optimal $\mathcal{O}(\epsilon^{-6})$ complexity of PASTA, the proof of which is in Appendix \ref{sec:pasta_weak}.

\begin{theorem}[Convergence of PASTA under Weak Convexity]\label{thm:pasta_weak}
Let Assumptions \ref{assump:oracle}, \ref{assump:weak_convexity}, and \ref{assump:lipschitz} hold. Let the sequence $\{x_s\}$ be generated by PASTA running for $S$ epochs of length $K$. Assume the coupling parameter satisfies $\lambda > \rho$ with strong convexity coefficient $\mu = \lambda - \rho > 0$, and define the anchoring sequence as $\beta_t = \lambda \eta$. Set the step size $\eta$, epoch count $S$, and epoch length $K$ as
\begin{equation}
\resizebox{\linewidth}{!}{$
\eta = \min\left\{ \frac{\mu^3}{96 \lambda^2 (3\lambda^2 - \rho^2)}, \frac{ \mu^3 \epsilon^2}{256 G^2 \lambda^2(3\lambda^2 - \rho^2)} \right\}, \; S = \left\lceil \frac{32 \lambda^2 \Delta}{\mu \epsilon^2} \right\rceil, \; K = \left\lceil 1 + \frac{1}{\eta\mu} \log\left( \frac{12(3\lambda^2 - \rho^2)}{\mu^2} \right) \right\rceil.
$}
\end{equation}
Set the initial batch size $N_0$ at $t=0$ and the constant inner loop batch size $N_t = N$ for $t \ge 1$ as
\begin{equation}
N_0 = \left\lceil \max\left\{ \left(\frac{4}{3} + \frac{2\mu^2}{3(3\lambda^2-\rho^2)}\right) B_v^2 \eta^2 S^2, \frac{16 b_v^2 \eta^2}{3 \lambda^2 \epsilon^2} \right\} \right\rceil
\end{equation}
and
\begin{equation}
N = \left\lceil \max\left\{ \left( \frac{96(3\lambda^2-\rho^2)}{\mu^3} + \frac{48}{\mu} \right) B_v^2 \eta S^2, \frac{128 b_v^2 \eta (3\lambda^2-\rho^2)}{\mu^3 \lambda^2 \epsilon^2}, \frac{72 B_v^2 \eta (3\lambda^2-\rho^2)}{\mu^3} \right\} \right\rceil,
\end{equation}
where $\Delta \ge f(x_0) - \inf_x f(x)$. Then, $\frac{1}{S}\sum_{s=1}^S \mathbb{E}[\|\nabla \varphi_{1/\lambda}(x_{s-1})\|^2] \le \epsilon^2$ and the total stochastic oracle complexity evaluates to $S(N_0 + K N) = \mathcal{O}(\epsilon^{-6})$.
\end{theorem}

\begin{remark}\label{rem:pasta}
    Traditional Halpern iteration relies on an uncoupled, decaying anchoring sequence (e.g., $\beta_t = \frac{1}{t+2}$). Analyzing such uncoupled dynamics yields harmonic summations that introduce a logarithmic dependence on $\epsilon^{-1}$ (see Theorem 3.6 in \cite{alacaoglu2025towards} for an example). By coupling $\beta_t = \lambda \eta$ and holding $\eta$ constant, the inner loop contracts linearly at a rate of $(1 - \eta \mu)$ without accumulating logarithmic artifacts. Further, since the subproblem is $\mu$-strongly convex, the distance to the exact subproblem minimizer decays exponentially. The inner loop length $K = \mathcal{O}(\epsilon^{-2})$ is selected to reduce the optimization error of the surrogate until it is dominated by the statistical floor.
Finally, to deal with the extra variance term $B_v^2\|x_t - x_0\|^2$, the batch sizes $N_0$ and $N$ scale with $S^2$. Intuitively, this makes sense as the maximum distance the iterates can traverse over $S$ epochs scales with $S$; hence, ensuring the extra variance term $B_v^2\|x_t - x_0\|^2$ remains bounded by a global constant requires batch sizes of $\mathcal{O}(S^2) = \mathcal{O}(\epsilon^{-4})$.
\end{remark}

\begin{remark}
By definition, any $L$-smooth function is $L$-weakly convex. To derive the convergence guarantees for the $L$-smooth functions using Theorem \ref{thm:pasta_weak}, we simply set $\rho = L$ and $\lambda = 2L$, which means $\mu = \lambda - \rho = L$ for the regularized surrogate. The proximal mapping $\hat{x} = \mathrm{prox}_{f/\lambda}(x)$ implies $\nabla f(\hat{x}) + \lambda(\hat{x} - x) = 0$, which by Lemma \ref{lem:moreau_gradient} means $\nabla \varphi_{1/\lambda}(x) = \lambda(x - \hat{x}) = \nabla f(\hat{x})$. Utilizing the $L$-smoothness of the objective function $f$ and the triangle inequality we have $\|\nabla f(x)\| \le \|\nabla f(x) - \nabla f(\hat{x})\| + \|\nabla f(\hat{x})\| \le L\|x - \hat{x}\| + \|\nabla \varphi_{1/\lambda}(x)\|$. Thus, $\|\nabla f(x)\| \le \frac{3}{2}\|\nabla \varphi_{1/2L}(x)\|$. Consequently, using $\tilde{\epsilon}^2 = \frac{4}{9}\epsilon^2$, Theorem \ref{thm:pasta_weak} implies the same $\mathcal{O}(\epsilon^{-6})$ for smooth functions which we established in Theorem \ref{thm:pasta_dynamic} using dynamic batching, but here using an epoch-based instantiation of PASTA.
\end{remark}
\section{Conclusion}
\label{sec:conclusion}
We resolved the oracle complexity of non-convex minimization under the unbounded variance characterized by the BG-0 condition, recently shown to be the weakest variance condition. By establishing  $\Omega(\epsilon^{-6})$ and $\Omega(\epsilon^{-4})$ lower bounds for smooth and mean-square smooth functions, we formally quantified the information-theoretic penalty incurred when going beyond (effectively) the bounded variance assumption. To bridge this gap, we studied PASTA and demonstrated that carefully coupling the anchoring parameters and step sizes induces a nice curvature that counteracts the unbounded noise allowed by BG-0 oracles. Our analyses showed PASTA, with the right instantiations, attains minimax optimal rates not only for smooth functions but also for weakly convex, star-convex, and Polyak-Lojasiewicz landscapes without relying on bounded domain assumptions.
\bibliographystyle{acm}
\bibliography{ref.bib}
\appendix
\section{Proofs of Lower Bounds in Section \ref{sec:lower_bounds}}\label{app:lower_bounds}
Before presenting the proofs of our lower bounds, we recall the framework of \cite{arjevani2023lower}.

\begin{definition}[Progress Function]
For any vector $y \in \mathbb{R}^T$ and a threshold $\alpha \ge 0$, the progress function $\textnormal{prog}_\alpha(y)$ identifies the highest coordinate index whose absolute magnitude exceeds the threshold. Formally, defining a dummy variable $y_0 \equiv \infty$, we have
\begin{equation}
    \textnormal{prog}_\alpha(y) := \max \left\{ i \in \{0, 1, \dots, T\} : |y_i| > \alpha \right\}.
\end{equation}
\end{definition}

\begin{definition}[Zero-Respecting Algorithms]
An optimization algorithm $\mathsf{A}$ is termed zero-respecting if it only explores the coordinate subspace spanned by the gradient feedback previously observed from the oracle. Formally, an algorithm interacting with an oracle $\mathsf{O}_F$ belongs to the class of zero-respecting algorithms, denoted $\mathcal{A}_{\textnormal{\textsf{zr}}}$, if for any generated query sequence $x^{(1)}, x^{(2)}, \dots$, the support  satisfies
\begin{equation}
\textnormal{supp}(x^{(t)}) \subseteq \bigcup_{i=1}^{t-1} \textnormal{supp}(\mathsf{O}_F(x^{(i)})) \quad \forall t \ge 1.
\end{equation}
\end{definition}

Following \cite{arjevani2023lower}, we choose the underlying deterministic function $F$ to be a hard zero-chain construction. For each dimension $T \in \mathbb{N}$, we define
\begin{equation}
\bar{F}_T(x) := -\Psi(1)\Phi(x_1) + \sum_{i=2}^{T}\left[\Psi(-x_{i-1})\Phi(-x_i) - \Psi(x_{i-1})\Phi(x_i)\right],\label{eq:f_unscaled}
\end{equation}
where the non-linear component functions $\Psi$ and $\Phi$ are constructed as
\begin{equation}
\Psi(t) = \left\{
\begin{array}{ll}
0,\quad&t\leq{}1/2,\\
\exp\left(1-\frac{1}{(2t-1)^{2}}\right),\quad&t>1/2.
\end{array}
\right.\quad\quad\text{and}\quad\quad\Phi(t) =
\sqrt{e}\int_{-\infty}^{t}e^{-\frac{1}{2}\eta^{2}}d\eta.\label{eq:psi_phi}
\end{equation}
The function $\bar{F}_T$ is a deterministic zero-chain. It inherently maintains a large gradient magnitude unless all coordinates have been activated and optimized to be sufficiently large (i.e., $\text{prog}_1(x)\geq{}T$). We formally enumerate its relevant geometric properties below.

The bottleneck of any zero-respecting algorithm operating on a probability-$p$ zero-chain is established by the expected hitting time to activate subsequent coordinates.

\begin{lemma}[Lemma 1 in \cite{arjevani2023lower}] \label{lem:prob-zero-chain}
Let $g(x,z)$ be a probability-$p$ zero-chain gradient estimator for $F:\mathbb{R}^{T}\to\mathbb{R}$, and let $\mathsf{O}$ be any oracle with $\mathsf{O}_{F}(x,z)=(F(x),g(x,z))$. Let $\{x^{(k)}_{\mathsf{A}}\}$ be the sequence of queries generated by any zero-respecting algorithm $\mathsf{A}\in\mathcal{A}_{\textnormal{\textsf{zr}}}$ interacting with $\mathsf{O}_F$. Then, with probability at least $1-\delta$, the  progress is  bounded by $\max_{k\in[K]}\textnormal{prog}_0( x^{(k)}_{\mathsf{A}} ) < T$ for any query budget $K \leq \frac{T-\log(1/\delta)}{2p}$.
\end{lemma}

\begin{lemma}[Lemma 2 in \cite{arjevani2023lower}] \label{lem:deterministic-construction}
The function $\bar{F}_T$ satisfies the following properties. (1) $\bar{F}_T(0) - \inf_{x}\bar{F}_T(x) \leq \Delta_0\cdot T$, where $\Delta_0 = 12$. (2) The gradient of $\bar{F}_T$ is $\ell_1$-Lipschitz continuous, where $\ell_1=152$. (3) For all $x\in\mathbb{R}^{T}$, $\|\nabla\bar{F}_T(x)\|_{\infty}\leq\gamma_\infty$, where $\gamma_\infty = 23$. (4) For all $x\in\mathbb{R}^T$, $\text{prog}_0(\nabla\bar{F}_T(x))\leq\text{prog}_{1/2}(x)+1$. (5) For all $x\in\mathbb{R}^T$, if $\text{prog}_{1}(x)<T$ then $\|\nabla\bar{F}_T(x)\| \ge |\nabla_{\text{prog}_{1}(x)+1}\bar{F}_T(x)| > 1$.
\end{lemma}

To simulate the difficulty of stochastic optimization, the deterministic chain is evaluated using a stochastic oracle that masks the true gradient.

A simple extension to Lemma~\ref{lem:deterministic-construction} utilizing the structure of $\bar{F}_T(x)$ in \eqref{eq:f_unscaled} to provide a total variation bound is as follows.

\begin{remark}[Total Variation Bound]
    \label{rem:total_variation}
    Let $\bar{F}_T$ be defined as in \eqref{eq:f_unscaled}, then for all $x \in \R^T$, we have $|\bar{F}_T(x)| \le e\sqrt{2\pi e}\,T < 12T$. Consequently, $\sup_{x}\bar{F}_T(x) \le 12T$ and $\inf_{x}\bar{F}_T(x) \ge -12T$, and hence $\sup_{x}\bar{F}_T(x)-\inf_{x}\bar{F}_T(x)\le 24T = 2\Delta_0\cdot T.$
\end{remark}

\begin{proof}[Proof of Remark~\ref{rem:total_variation}]
    Recall the definition of $\Psi(t)$ and $\Phi(t)$ from \eqref{eq:psi_phi}. It then immediately follows that $0 \le \Psi(t) \le e$ and $0 \le \Phi(t) \le \sqrt{e}\int_{-\infty}^{\infty} e^{-\eta^2/2}\,d\eta = \sqrt{2\pi e}$. Let $M:=\sqrt{2\pi e}$. Now, we focus on $\bar{F}_T(x)$. We bound the first term first.

    Since $\Psi(1) = 1$, we have 
    $$\left|-\Psi(1)\Phi(x_1)\right| = |\Phi(x_1)| \le M.$$
    Next, fix any $i \in \{2,\dots,T\}$ and consider the summand $\Psi(-x_{i-1})\Phi(-x_i)-\Psi(x_{i-1})\Phi(x_i)$. Because $\Psi(t)$ is nonzero only when $t>\frac{1}{2}$, at most one of $\Psi(a)$ and $\Psi(-a)$ can be nonzero for any $a \in \mathbb{R}$. Indeed, if both were nonzero, then one would simultaneously have $a>\frac{1}{2}$ and $-a>\frac{1}{2}$, which is impossible. Therefore, for each $i \ge 2$, at most one of the two terms in the summand is active. It follows that
    $$\left| \Psi(-x_{i-1})\Phi(-x_i)-\Psi(x_{i-1})\Phi(x_i) \right| \le eM.$$

    Summing this all, we get $|\bar{F}_T(x)| \leq M + (T-1)eM \leq TeM \leq 12T.$ The rest just follows as presented in the remark above.
\end{proof}

\begin{lemma}[Lemma 3 in \cite{arjevani2023lower}] \label{lem:pzc-basic}
The stochastic gradient estimator $\bar{g}_T$ is a probability-$p$ zero-chain, is unbiased for $\nabla \bar{F}_T$, and has variance $\mathbb{E} [\|\bar{g}_T(x,z) - \nabla \bar{F}_T(x)\|^2] \le \varsigma^2 \cdot \frac{1-p}{p}$ for all $x\in\mathbb{R}^T$, where $\varsigma=23$.
\end{lemma}

To establish limits under the MSS condition, we rely on an advanced zero-chain estimator that reduces the variance of gradient differences.

\begin{lemma}[Lemma 4 in \cite{arjevani2023lower}] \label{lem:pzc-mss}
There exists a stochastic gradient estimator $\bar{g}_T(x, z)$ for $\bar{F}_T(x)$ that is a probability-$p$ zero-chain, is unbiased for $\nabla \bar{F}_T$, and satisfies the variance bound $\mathbb{E} [\|\bar{g}_T(x,z) - \nabla \bar{F}_T(x)\|^2] \le \varsigma^2 \frac{1-p}{p}$ with $\varsigma=23$, and the mean-square smoothness property $\mathbb{E} [\|\bar{g}_T(x,z) - \bar{g}_T(y,z)\|^2] \le \frac{\bar{\ell}_1^2}{p} \|x - y\|^2$ for all $x, y \in \mathbb{R}^T$, where $\bar{\ell}_1 = 328$.
\end{lemma}

Finally, we state a result which is an extension of classical mean estimation \cite{lecam1973convergence,yu1997assouad} that we need for the MS lower bound.
\begin{lemma}[Lemma 10 in \cite{arjevani2023lower}] \label{lem:mean_estimation}
There exists a number $c_0>0$ such that for any dimension $T$ and $\epsilon \leq \sqrt{\bar{L} \Delta}$, extracting an $\epsilon$-stationary mean vector from an oracle with variance $\sigma^2$ requires
$\geq c_0 \sigma^2 \epsilon^{-2}$ for some function with initial suboptimality $\Delta$ and smoothness constant $\bar{L}$.
\end{lemma}
\subsection{Proof of Theorem \ref{thm:bg0_lower_bound}}\label{sec:bg0_lower_bound}
\begin{proof}[Proof of Theorem \ref{thm:bg0_lower_bound}]
Let $D = \frac{\Delta}{4\epsilon}$. We define a coordinate function $f_0 \colon \mathbb{R} \to \mathbb{R}$ satisfying 
\begin{equation}\label{eq:travel}
f_0(0) = 0,\qquad f_0'(u) = 
\begin{cases} 
-2\epsilon, & u \le D \\ 
-2\epsilon + \frac{L}{2}(u - D), & u \in (D, D + \frac{4\epsilon}{L}) \\ 
0, & u \ge D + \frac{4\epsilon}{L}. 
\end{cases}
\end{equation}
Integrating this derivative establishes $\inf_u f_0(u) \ge -\frac{\Delta}{2} - \frac{4\epsilon^2}{L}$. We define a twice-differentiable activation function $\phi \colon \mathbb{R} \to [0, 1]$ such that 
\begin{equation}\label{eq:activation}
  \phi(u) =  
\begin{cases}  
0& u \le D/2 \\ 
1& u \le D \\
10t^3 - 15t^4 + 6t^5& D/2 \le u \le D,
\end{cases}
\qquad t = \frac{8\epsilon u}{\Delta} - 1.
\end{equation}
Evaluating the global extrema over this polynomial yields the derivative bounds $\|\phi'\|_{\infty} = \frac{15\epsilon}{\Delta}$ and $\|\phi''\|_{\infty} = \frac{640\epsilon^2}{\sqrt{3}\Delta^2}$.

Following Lemma \ref{lem:deterministic-construction}, we construct a scaled and globally shifted high-dimensional zero-chain $F_{\text{scaled}} \colon \mathbb{R}^T \to \mathbb{R}$
\begin{equation}\label{eq:f_scaled}
F_{\text{scaled}}(y) = \frac{L \lambda^2}{2 \ell_1} \left( \bar{F}_T\left(\frac{y}{\lambda}\right) - \sup_{x \in \mathbb{R}^T} \bar{F}_T(x) \right), \quad \lambda = \frac{4 \ell_1 \epsilon}{L}, \quad T = \left\lfloor \frac{L \Delta}{768 \ell_1 \epsilon^2} \right\rfloor
\end{equation}
where $\ell_1 = 152$. By Remark~\ref{rem:total_variation}, the total variation of the base zero-chain is bounded by $2\Delta_0 \cdot T$. Applying the restriction on $T$ ensures $-\frac{\Delta}{4} \le F_{\text{scaled}}(y) \le 0$ globally. The unscaled gradient satisfies $\|\nabla \bar{F}_T(x)\|_{\infty} > 1$ for $\text{prog}_1(x) < T$. The scaled gradient expands to $\nabla F_{\text{scaled}}(y) = 2\epsilon \nabla \bar{F}_T(y/\lambda)$, thereby preserving $\|\nabla F_{\text{scaled}}(y)\|_{\infty} > 2\epsilon$ for $\text{prog}_1(y/\lambda) < T$.

We define a composite objective function $F \colon \mathbb{R} \times \mathbb{R}^T \to \mathbb{R}$ that couples the coordinate function with the zero-chain using the activation function according to
\begin{equation}
F(u, y) = f_0(u) + \phi(u) F_{\text{scaled}}(y).
\end{equation}
The initialization is set to $x_0 = (0, 0)$. It is then immediate that $\inf_{u,y} F(u,y) \ge -\frac{3\Delta}{4} - \frac{4\epsilon^2}{L}$. Imposing  $\epsilon \le \sqrt{\frac{L\Delta}{1536\ell_1}}$ ensures $\frac{4\epsilon^2}{L} \le \frac{\Delta}{384\ell_1} \le \frac{\Delta}{4}$. Thus, $F(x_0) - \inf_{u,y} F(u,y) \le \Delta$.

The composite Hessian matrix $\nabla^2 F(u, y)$ decomposes into a block-diagonal operator $H_{\text{diag}}$ and an off-diagonal operator $H_{\text{off}}$. The spectral norm of $H_{\text{diag}}$ is bounded by $\frac{L}{2} + \|\phi''\|_{\infty} |F_{\text{scaled}}(y)|$. Substituting the polynomial and chain bounds yields $\|H_{\text{diag}}\|_2 \le \frac{L}{2} + \frac{160\epsilon^2}{\sqrt{3}\Delta}$. The spectral norm of the symmetric off-diagonal block $H_{\text{off}}$ equates to the Euclidean norm of $\phi'(u) \nabla F_{\text{scaled}}(y)$. By Lemma \ref{lem:deterministic-construction}, $\|\nabla \bar{F}_T(x)\|_{\infty} \le 23$, and consequently using a norm bound $\|\nabla \bar{F}_T(x)\|_2 \le 23\sqrt{T}$. Consequently, $\|\nabla F_{\text{scaled}}(y)\|_2 \le 46\epsilon\sqrt{T} \le \frac{46}{\sqrt{768\ell_1}}\sqrt{L\Delta}$. Thus, $\|H_{\text{off}}\|_2 \le \frac{690\epsilon}{\sqrt{768\ell_1\Delta}}\sqrt{L}$. By Weyl's inequality
\begin{equation}
\|\nabla^2 F(u, y)\|_2 \le \frac{L}{2} + \frac{160\epsilon^2}{\sqrt{3}\Delta} + \frac{690\epsilon}{\sqrt{768\ell_1\Delta}}\sqrt{L}.
\end{equation}
Since we impose $\epsilon \le \sqrt{\frac{L\Delta}{1536\ell_1}}$ we have $\|\nabla^2 F(u, y)\|_2 \le L$, certifying global $L$-smoothness.

The expanded composite gradient vector evaluates to
\begin{equation}
\nabla F(u, y) = \begin{pmatrix} f_0'(u) + \phi'(u) F_{\text{scaled}}(y) \\ \phi(u) \nabla F_{\text{scaled}}(y) \end{pmatrix}.
\end{equation}
For any $u \le D$, $f_0'(u) = -2\epsilon$. The strict non-negativity of $\phi'(u)$ combined with the globally non-positive zero-chain $F_{\text{scaled}}(y) \le 0$ enforces $\phi'(u) F_{\text{scaled}}(y) \le 0$. The first gradient coordinate satisfies $f_0'(u) + \phi'(u) F_{\text{scaled}}(y) \le -2\epsilon$, implying $\|\nabla F(u, y)\| \ge 2\epsilon > \epsilon$. Identifying any $\epsilon$-stationary point then requires progressing to $u > D/2$. At these locations
\begin{equation}
\|x - x_0\|^2 \ge u^2 > \frac{\Delta^2}{64\epsilon^2}.
\end{equation}
At these locations, under Assumption \ref{assump:oracle}, the BG-0 stochastic oracle variance evaluates to
\begin{equation}
\tilde{\sigma}^2 = B_v^2 \frac{\Delta^2}{64\epsilon^2} + b_v^2.
\end{equation}

We deploy a composite stochastic gradient estimator $G(u, y, z)$ incorporating the zero-chain estimator $\bar{g}_T$ from Lemma \ref{lem:pzc-basic}
\begin{equation}
G(u, y, z) = \begin{pmatrix} f_0'(u) + \phi'(u) F_{\text{scaled}}(y) \\ \phi(u) 2\epsilon \bar{g}_T(y/\lambda, z) \end{pmatrix}.
\end{equation}
The variance of this estimator satisfies $\mathbb{E} [\|G(u, y, z) - \nabla F(u, y)\|^2] \le \frac{4 \varsigma^2 \epsilon^2 (1-p)}{p}$. Setting this upper bound equal to $\tilde{\sigma}^2$ yields
\begin{equation}\label{eq:def_p}
p = \frac{256 \varsigma^2 \epsilon^4}{B_v^2 \Delta^2 + 64 b_v^2 \epsilon^2 + 256 \varsigma^2 \epsilon^4}.
\end{equation}

By Lemma \ref{lem:prob-zero-chain}, determining a coordinate satisfying $\text{prog}_1(y/\lambda) \ge T$ demands an expected query complexity bounded strictly below by $\frac{T-1}{2p}$. The constraint $\epsilon \le \sqrt{\frac{L\Delta}{1536\ell_1}}$ along with the definition of $T$ ensures $T \ge 2$, such that  $\frac{T-1}{2p} \ge \frac{T}{4p}$. Substituting the definition of $p$ yields
\begin{equation}
\frac{T}{4p} = \frac{T}{4} \left( 1 + \frac{B_v^2 \Delta^2 + 64 b_v^2 \epsilon^2}{256 \varsigma^2 \epsilon^4} \right) \ge \frac{T(B_v^2 \Delta^2 + 64 b_v^2 \epsilon^2)}{1024 \varsigma^2 \epsilon^4}.
\end{equation}
Applying the dimensional bound $T \ge \frac{L \Delta}{1536 \ell_1 \epsilon^2}$ leads to the final expected query complexity
\begin{equation}
\frac{T}{4p} \ge \frac{B_v^2 L \Delta^3}{1572864 \ell_1 \varsigma^2 \epsilon^6} + \frac{b_v^2 L \Delta}{24576 \ell_1 \varsigma^2 \epsilon^4}.
\end{equation}

To apply this bound for all randomized algorithms, following \cite{arjevani2023lower}, we embed the construction within a randomized orthogonal transformation $V = \text{diag}(1, U)$. Operating on the rotated objective $F(V^\top x)$ preserves Euclidean geometry entirely, directly extending the zero-respecting sequence bound, thereby furnishing the proof.
\end{proof}
\subsection{Proof of Theorem \ref{thm:bg0_mss_lower_bound}}\label{sec:bg0_mss_lower_bound}
\begin{proof}[Proof of Theorem \ref{thm:bg0_mss_lower_bound}]
The proof is almost identical to that of Theorem \ref{thm:bg0_lower_bound}, but we have to use Lemma \ref{lem:pzc-mss} and a suitable reparameterization of the deterministic smooth constant as done in \cite{arjevani2023lower}. Additionally, we note that the restriction on $B_v$ is essential to ensure $T \ge 2$ holds, something we need in our construction. We rewrite the near-identical calculation here for the ease of the reader.

Let $D = \frac{\Delta}{4\epsilon}$. We define a coordinate function $f_0 \colon \mathbb{R} \to \mathbb{R}$ satisfying $f_0(0) = 0$ with the derivative
\begin{equation}
f_0'(u) = 
\begin{cases} 
-2\epsilon, & u \le D \\ 
-2\epsilon + \frac{L}{2}(u - D), & u \in (D, D + \frac{4\epsilon}{L}) \\ 
0, & u \ge D + \frac{4\epsilon}{L}. 
\end{cases}
\end{equation}
Integrating this derivative establishes $\inf_u f_0(u) \ge -\frac{\Delta}{2} - \frac{4\epsilon^2}{L}$. We define a twice-differentiable activation function $\phi \colon \mathbb{R} \to [0, 1]$ such that $\phi(u) = 0$ for $u \le D/2$ and $\phi(u) = 1$ for $u \ge D$. Over the transition interval $u \in (D/2, D)$, we define the normalization $t = \frac{8\epsilon u}{\Delta} - 1$ and set $\phi(u) = 10t^3 - 15t^4 + 6t^5$. Evaluating the global extrema over this polynomial yields the derivative bounds $\|\phi'\|_{\infty} = \frac{15\epsilon}{\Delta}$ and $\|\phi''\|_{\infty} = \frac{640\epsilon^2}{\sqrt{3}\Delta^2}$.

At locations such that $u > D/2$ we have
\begin{equation}
\|x - x_0\|^2 \ge u^2 > \frac{\Delta^2}{64\epsilon^2}.
\end{equation}
At these locations, under Assumption \ref{assump:oracle}, the BG-0 stochastic oracle variance evaluates to
\begin{equation}
\tilde{\sigma}^2 = B_v^2 \frac{\Delta^2}{64\epsilon^2} + b_v^2.
\end{equation}

We deploy a stochastic gradient estimator $G(u, y, z)$ incorporating the estimator $\bar{g}_T$ from Lemma \ref{lem:pzc-mss}
\begin{equation}
G(u, y, z) = \begin{pmatrix} f_0'(u) + \phi'(u) F_{\text{scaled}}(y) \\ \phi(u) 2\epsilon \bar{g}_T(y/\lambda, z) \end{pmatrix}.
\end{equation}
The variance of the coupled stochastic gradient $G(u, y, z)$ is bounded by $\frac{4 \varsigma^2 \epsilon^2 (1-p)}{p}$. Setting this variance to $\tilde{\sigma}^2$ yields
\begin{equation} \label{eq:prob_p_mss}
p = \frac{256 \varsigma^2 \epsilon^4}{B_v^2 \Delta^2 + 64 b_v^2 \epsilon^2 + 256 \varsigma^2 \epsilon^4}.
\end{equation}

To satisfy the global mean-squared smoothness condition parameterized by $\bar{L}$ in Assumption \ref{assump:ms-smooth}, we parameterize the effective deterministic smoothness $L$  in Lemma \ref{lem:pzc-mss} according to $L = \frac{\ell_1}{\bar{\ell}_1} \bar{L} \sqrt{p}.$
Following Lemma \ref{lem:deterministic-construction}, we construct the scaled and globally shifted zero-chain $F_{\text{scaled}} \colon \mathbb{R}^T \to \mathbb{R}$ using this modified deterministic smoothness $L$
\begin{equation} \label{eq:f_scaled_mss}
F_{\text{scaled}}(y) = \frac{L \lambda^2}{2 \ell_1} \left( \bar{F}_T\left(\frac{y}{\lambda}\right) - \sup_{x \in \mathbb{R}^T} \bar{F}_T(x) \right), \quad \lambda = \frac{4 \ell_1 \epsilon}{L}, \quad T = \left\lfloor \frac{\bar{L} \Delta \sqrt{p}}{768 \bar{\ell}_1 \epsilon^2} \right\rfloor.
\end{equation}
By Remark~\ref{rem:total_variation}, the total variation of the base zero-chain is bounded by $2\Delta_0 \cdot T$. The definition of  $T$ ensures $-\frac{\Delta}{4} \le F_{\text{scaled}}(y) \le 0$  globally. For the unscaled gradient we have $\|\nabla \bar{F}_T(x)\|_{\infty} > 1$ for $\text{prog}_1(x) < T$. Thus, $\nabla F_{\text{scaled}}(y) = 2\epsilon \nabla \bar{F}_T(y/\lambda)$, preserving $\|\nabla F_{\text{scaled}}(y)\|_{\infty} > 2\epsilon$ for $\text{prog}_1(y/\lambda) < T$.

The composite objective function $F \colon \mathbb{R} \times \mathbb{R}^T \to \mathbb{R}$ couples the coordinate sequence with the zero-chain
\begin{equation}
F(u, y) = f_0(u) + \phi(u) F_{\text{scaled}}(y).
\end{equation}
The initialization is set to $x_0 = (0, 0)$. Also, $\inf_{u,y} F(u,y) \ge -\frac{3\Delta}{4} - \frac{4\epsilon^2}{L}$. The  bound on $\epsilon$ leads to $\epsilon^2 \le \frac{L\Delta}{1536\ell_1}$ such that along with the constraint on $T$ we have $\frac{4\epsilon^2}{L} \le \frac{\Delta}{384\ell_1} \le \frac{\Delta}{4}$. Thus, $F(x_0) - \inf_{u,y} F(u,y) \le \Delta$. Furthermore, substituting $\epsilon^2 \le \frac{L\Delta}{1536\ell_1}$ bounds the spectral perturbations of the composite Hessian exactly as in Theorem \ref{thm:bg0_lower_bound}, ensuring $\|\nabla^2 F(u, y)\|_2 \le L$ and deterministic $L$-smoothness.

The expanded composite gradient vector evaluates to
\begin{equation}
\nabla F(u, y) = \begin{pmatrix} f_0'(u) + \phi'(u) F_{\text{scaled}}(y) \\ \phi(u) \nabla F_{\text{scaled}}(y) \end{pmatrix}.
\end{equation}
For any $u \le D$, $f_0'(u) = -2\epsilon$. The strict non-negativity of $\phi'(u)$ combined with the globally non-positive zero-chain $F_{\text{scaled}}(y) \le 0$ enforces $\phi'(u) F_{\text{scaled}}(y) \le 0$. The first gradient coordinate satisfies $f_0'(u) + \phi'(u) F_{\text{scaled}}(y) \le -2\epsilon$, implying $\|\nabla F(u, y)\| \ge 2\epsilon > \epsilon$. Identifying any $\epsilon$-stationary point, therefore, mathematically necessitates progressing to $u > D/2$.

By Lemma \ref{lem:prob-zero-chain}, determining a coordinate satisfying $\text{prog}_1(y/\lambda) \ge T$ requires an expected query complexity bounded strictly below by $\frac{T-1}{2p}$. The restriction on $\epsilon$ ensures $T \ge 2$, establishing  $\frac{T-1}{2p} \ge \frac{T}{4p}$. Substituting the dimension $T$ yields
\begin{equation} \label{eq:complexity_mss_expansion}
\frac{T}{4p} = \frac{1}{4p} \left\lfloor \frac{\bar{L} \Delta \sqrt{p}}{768 \bar{\ell}_1 \epsilon^2} \right\rfloor \ge \frac{\bar{L} \Delta}{6144 \bar{\ell}_1 \epsilon^2 \sqrt{p}}.
\end{equation}
Noting $\frac{1}{\sqrt{p}} \ge \frac{\sqrt{B_v^2 \Delta^2 + 64 b_v^2 \epsilon^2}}{16 \varsigma \epsilon^2}$ and substituting it into \eqref{eq:complexity_mss_expansion} yields
\begin{equation} \label{eq:complexity_mss_final}
\frac{T}{4p} \ge \frac{\bar{L} \Delta \sqrt{B_v^2 \Delta^2 + 64 b_v^2 \epsilon^2}}{98304 \bar{\ell}_1 \varsigma \epsilon^4} \ge \frac{B_v \bar{L} \Delta^2}{98304 \sqrt{2} \bar{\ell}_1 \varsigma \epsilon^4} + \frac{b_v \bar{L} \Delta}{12288 \sqrt{2} \bar{\ell}_1 \varsigma \epsilon^3}
\end{equation}
where we used $\sqrt{A+B} \ge \frac{1}{\sqrt{2}}(\sqrt{A} + \sqrt{B})$.

By Lemma \ref{lem:mean_estimation}, extracting an $\epsilon$-stationary mean vector from an oracle with variance $\tilde{\sigma}^2$ requires $c_0\frac{\tilde{\sigma}^2}{\epsilon^2}$ independent queries as long as $\epsilon$ satisfies the imposed upper bound. Thus, when  $u > D/2$, this evaluates to
\begin{equation} \label{eq:stat_complexity}
\frac{\tilde{\sigma}^2}{\epsilon^2} = \frac{B_v^2 \Delta^2}{64 \epsilon^4} + \frac{b_v^2}{\epsilon^2}.
\end{equation}

Integrating the last bound with the expected query complexity lower bound on $\frac{T}{4p}$, in a near-identical fashion as done in \cite{arjevani2023lower} leads to the lower bound. Lifting this complexity to the universal class of all randomized algorithms is done exactly as explained in the proof of Theorem \ref{thm:bg0_lower_bound} by following the randomized rotation idea of \cite{arjevani2023lower}.
\end{proof}
\section{Proofs of Upper Bounds in Section \ref{sec:upper_bounds}}
\subsection{Proof of Theorem \ref{thm:pasta_dynamic}}\label{sec:pasta_dynamic}
\begin{proof}[Proof of Theorem \ref{thm:pasta_dynamic}]
Setting $S=1$, $\beta_t = 0$, with $p_t = 1$ simplifies  PASTA to $x_{t+1} = x_t - \eta g_t$. The mini-batch estimator $g_t$ is evaluated utilizing a dynamic batch size $N_t$
\begin{equation}
N_t = \max\left\{1,\left\lceil \frac{B_v^2 \|x_t - x_0\|^2 + b_v^2}{\sigma^2} \right\rceil\right\}
\end{equation}
where $\sigma^2 > 0$ is some target variance independent of $\epsilon$. By the independence of the stochastic queries conditioned on the filtration at $t$ and the $\mathcal{F}_{t-1}$-measurability of $x_t$
\begin{equation}
\mathbb{E}_t[\|g_t - \nabla f(x_t)\|^2] \le \frac{B_v^2 \|x_t - x_0\|^2 + b_v^2}{N_t} \le \sigma^2.
\end{equation}
The next few steps are standard. Applying $L$-smoothness
\begin{equation}
f(x_{t+1}) \le f(x_t) - \eta \langle \nabla f(x_t), g_t \rangle + \frac{L \eta^2}{2} \|g_t\|^2.
\end{equation}
Taking the conditional expectation and invoking unbiasedness
\begin{equation}
\mathbb{E}_t[f(x_{t+1})] \le f(x_t) - \eta \|\nabla f(x_t)\|^2 + \frac{L \eta^2}{2} \mathbb{E}_t\bigl[\|g_t\|^2\bigr].
\end{equation}
Expanding the second moment and utilizing the variance bound above,
\begin{equation}
\mathbb{E}_t[f(x_{t+1})] \le f(x_t) - \eta \left( 1 - \frac{L \eta}{2} \right) \|\nabla f(x_t)\|^2 + \frac{L \eta^2 \sigma^2}{2}.
\end{equation}
Since $\eta \le \frac{1}{L}$, we have $1 - \frac{L \eta}{2} \ge \frac{1}{2}$. Taking the total expectation, recursively unrolling from $t=0$ to $K-1$, and rearranging
\begin{equation} \label{eq:grad_sum_bound}
\sum_{t=0}^{K-1} \mathbb{E}[\|\nabla f(x_t)\|^2] \le \frac{2(f(x_0) - \mathbb{E}[f(x_{K})])}{\eta} + L \eta K\sigma^2 \le \frac{2\Delta}{\eta} + K L \eta \sigma^2.
\end{equation}
Recall that $\eta \leq \frac{\epsilon^2}{2 L \sigma^2}$ and $K = \left\lceil \frac{4 \Delta}{\eta \epsilon^2} \right\rceil$. 

To find the total oracle complexity, we need to estimate $\|x_t - x_0\|$, which determines the batch size $N_t$.
Using PASTA's update rule
\begin{equation}
x_t - x_0 = -\eta \sum_{j=0}^{t-1} \nabla f(x_j) - \eta \sum_{j=0}^{t-1} (g_j - \nabla f(x_j)).
\end{equation}
Using $\|a + b\|^2 \le 2\|a\|^2 + 2\|b\|^2$ and taking the expectation
\begin{equation}
\mathbb{E}[\|x_t - x_0\|^2] \le 2 \eta^2 \mathbb{E}\Bigl[ \bigl\| \sum_{j=0}^{t-1} \nabla f(x_j) \bigr\|^2\Bigr] + 2 \eta^2 \mathbb{E} \Bigl[ \bigl\| \sum_{j=0}^{t-1} (g_j - \nabla f(x_j)) \bigr\|^2\Bigr].
\end{equation}
The first term is bounded by $t \sum_{j=0}^{t-1} \|\nabla f(x_j)\|^2$. Substituting the bound from \eqref{eq:grad_sum_bound}
\begin{equation}
2 \eta^2 \mathbb{E} \Bigl[ \bigl\| \sum_{j=0}^{t-1} \nabla f(x_j) \bigr\|^2\Bigr] \le 2 \eta^2 K \sum_{j=0}^{K-1} \mathbb{E}\bigl[\|\nabla f(x_j)\|^2\bigr] \le 4 \eta K \Delta + 2 K^2 L \eta^3 \sigma^2.
\end{equation}
On the other hand, by the unbiasedness of the oracle, the second term can be bounded by
\begin{equation}
2 \eta^2 \mathbb{E} \Bigl[ \bigl\| \sum_{j=0}^{t-1} (g_j - \nabla f(x_j)) \bigr\|^2\Bigr] = 2 \eta^2 \sum_{j=0}^{t-1} \mathbb{E}\bigl[\|g_j - \nabla f(x_j)\|^2\bigr] \le 2 \eta^2 K \sigma^2.
\end{equation}
Aggregating these two  bounds yields
\begin{equation} \label{eq:_drift_bound}
\mathbb{E}\bigl[\|x_t - x_0\|^2\bigr] \le 4 \eta K \Delta + 2 K^2 L \eta^3 \sigma^2 + 2 \eta^2 K \sigma^2.
\end{equation}
To find an $\epsilon$-stationary solution, note from \eqref{eq:grad_sum_bound} 
\begin{equation}
\frac{1}{K} \sum_{t=0}^{K-1} \mathbb{E}\bigl[\|\nabla f(x_t)\|^2\bigr] \le \frac{2\Delta}{K \eta} + L \eta \sigma^2 \le \epsilon^2.
\end{equation}
It remains to bound the total stochastic oracle complexity. We consider the two possible branches of the stepsize.

\noindent\textbf{Case 1:} $\eta = \frac{\epsilon^2}{2L\sigma^2}$, equivalently $\epsilon^2 \le 2\sigma^2$.

In this case, since $K = \left\lceil \frac{4\Delta}{\eta\epsilon^2}\right\rceil$, the bound \eqref{eq:_drift_bound} simplifies to
\begin{equation}
\mathbb{E}\bigl[\|x_t - x_0\|^2\bigr]
\le
\frac{32 \Delta^2}{\epsilon^2} + \frac{4 \Delta}{L} + \frac{6 \Delta \epsilon^2}{L \sigma^2} + \frac{\epsilon^4}{2 L^2 \sigma^2} + \frac{\epsilon^6}{4 L^2 \sigma^4}.
\end{equation}

The total stochastic oracle complexity integrates the dynamic batch sizes across the entire optimization path. Taking the total expectation,
\begin{equation}
\mathbb{E}[N_{total}]
=
\sum_{t=0}^{K-1} \mathbb{E}[N_t]
\le
\sum_{t=0}^{K-1} \left( 1 + \frac{B_v^2 \mathbb{E}\bigl[\|x_t - x_0\|^2\bigr] + b_v^2}{\sigma^2} \right).
\end{equation}
Substituting the bound on $\mathbb{E}\bigl[\|x_t - x_0\|^2\bigr]$ found above yields
\begin{equation}\label{eq:exact_bound_dynamic}
    \begin{aligned}
        \mathbb{E}[N_{total}] &\le \frac{256 B_v^2 \Delta^3 L}{\epsilon^6} + \frac{80 B_v^2 \Delta^2}{\epsilon^2 \sigma^2} + \frac{32 B_v^2 \Delta^2}{\epsilon^4} + \frac{8 B_v^2 \Delta \epsilon^2}{L \sigma^4} + \frac{8 B_v^2 \Delta}{L \sigma^2} 
        \\
        &\quad + \frac{B_v^2 \epsilon^6}{4 L^2 \sigma^6} + \frac{B_v^2 \epsilon^4}{2 L^2 \sigma^4} + \frac{8 \Delta L b_v^2}{\epsilon^4} + \frac{8 \Delta L \sigma^2}{\epsilon^4} + \frac{b_v^2}{\sigma^2} + 1,
    \end{aligned}
\end{equation}
and therefore $\mathbb{E}[N_{\mathrm{total}}] = \mathcal{O}(\epsilon^{-6})$.

\noindent\textbf{Case 2:} $\eta = \frac{1}{L}$, equivalently $\epsilon^2 > 2\sigma^2$.
In this case, $ K = \left\lceil \frac{4\Delta L}{\epsilon^2} \right\rceil = \mathcal{O}(\epsilon^{-2}).$
Using \eqref{eq:_drift_bound},
\begin{equation}
    \mathbb{E}\bigl[\|x_t - x_0\|^2\bigr] \le \frac{4K\Delta}{L} + \frac{2K^2\sigma^2}{L^2} + \frac{2K\sigma^2}{L^2}
    = \mathcal{O}(\epsilon^{-4}).
\end{equation}

Therefore,
\begin{equation}
    \begin{aligned}
        \mathbb{E}[N_{total}] &=  \sum_{t=0}^{K-1} \mathbb{E}[N_t]
        \le \sum_{t=0}^{K-1} \left( 1 + \frac{B_v^2 \mathbb{E}\bigl[\|x_t - x_0\|^2\bigr] + b_v^2}{\sigma^2} \right) 
        \\
        &\le K\left(1 + \frac{b_v^2}{\sigma^2}\right) + \frac{B_v^2}{\sigma^2} \sum_{t=0}^{K-1} \mathbb{E}\bigl[\|x_t - x_0\|^2\bigr] 
        = \mathcal{O}(\epsilon^{-2}) + \mathcal{O}(\epsilon^{-2})\mathcal{O}(\epsilon^{-4})
        = \mathcal{O}(\epsilon^{-6}).
    \end{aligned}
\end{equation}

Combining the two cases, we conclude that
$
\mathbb{E}[N_{total}] = \mathcal{O}(\epsilon^{-6}),
$
which finishes the proof.
\end{proof}
\subsection{Proof of Theorem \ref{thm:pasta_dynamic2}}\label{sec:pasta_dynamic_cvx}
\begin{proof}[Proof of Theorem \ref{thm:pasta_dynamic2}]
Under the parameterization $S=1$, $p_t=1$, and $\beta_t=0$, PASTA reduces to $x_{t+1} = x_t - \eta g_t$. The dynamic batch size $N_t$ and the independence of the oracle queries conditioned on $\mathcal{F}_t$ ensure
\begin{equation}
\mathbb{E}_t\bigl[\|g_t - \nabla f(x_t)\|^2\bigr] \le \sigma^2.
\end{equation}
Expanding the distance to the optimum and taking the conditional expectation yields
\begin{equation}
\mathbb{E}_t\bigl[\|x_{t+1} - x^*\|^2\bigr] = \|x_t - x^*\|^2 - 2\eta\langle \nabla f(x_t), x_t - x^* \rangle + \eta^2\mathbb{E}_t\bigl[\|g_t\|^2\bigr].
\end{equation}
The convexity of $f$ means $\langle \nabla f(x_t), x_t - x^* \rangle \ge f(x_t) - f(x^*)$. Also, by $L$-smoothness we have $\|\nabla f(x_t)\|^2 \le 2L(f(x_t) - f(x^*))$. Therefore,
\begin{equation}
\mathbb{E}_t\bigl[\|g_t\|^2\bigr] \le 2L(f(x_t) - f(x^*)) + \sigma^2.
\end{equation}
Substituting these relations into the distance expansion yields
\begin{equation}\label{eq:cvx-smooth}
\mathbb{E}_t\bigl[\|x_{t+1} - x^*\|^2\bigr] \le \|x_t - x^*\|^2 - 2\eta(1 - L\eta)(f(x_t) - f(x^*)) + \eta^2\sigma^2.
\end{equation}
Since $\eta \le \frac{1}{2L}$,  we have $1 - L\eta \ge \frac{1}{2}$. Taking the total expectation and rearranging
\begin{equation}
\eta\mathbb{E}[f(x_t) - f(x^*)] \le \mathbb{E}\bigl[\|x_t - x^*\|^2\bigr] - \mathbb{E}\bigl[\|x_{t+1} - x^*\|^2\bigr] + \eta^2\sigma^2.
\end{equation}
Summing from $t=0$ to $K-1$ and discarding the non-positive terms yields
\begin{equation}
\sum_{t=0}^{K-1} \mathbb{E}[f(x_t) - f(x^*)] \le \frac{\|x_0 - x^*\|^2}{\eta} + K\eta\sigma^2.
\end{equation}
Dividing by $K$, by convexity, and noting the definition of the ergodic average $\bar{x}_K$
\begin{equation}
\mathbb{E}[f(\bar{x}_K) - f(x^*)] \le \frac{\|x_0 - x^*\|^2}{K\eta} + \eta\sigma^2.
\end{equation}
Since $\eta \leq \frac{\epsilon}{2\sigma^2}$ and $K = \left\lceil \frac{2\|x_0 - x^*\|^2}{\eta \epsilon} \right\rceil$ guarantees $\mathbb{E}[f(\bar{x}_K) - f(x^*)] \le \epsilon$.

To determine the expected total stochastic oracle complexity, recall \eqref{eq:cvx-smooth} and discard the non-positive suboptimality term to bound $\mathbb{E}\bigl[\|x_{t+1} - x^*\|^2\bigr] \le \mathbb{E}\bigl[\|x_t - x^*\|^2\bigr] + \eta^2\sigma^2$. Unrolling this  establishes $\mathbb{E}\bigl[\|x_t - x^*\|^2\bigr] \le \|x_0 - x^*\|^2 + t\eta^2\sigma^2$. Applying the inequality $\|a + b\|^2 \le 2\|a\|^2 + 2\|b\|^2$ bounds
\begin{equation}
\mathbb{E}\bigl[\|x_t - x_0\|^2\bigr] \le 4\|x_0 - x^*\|^2 + 2t\eta^2\sigma^2.
\end{equation}
The total expected stochastic oracle complexity amounts to
\begin{equation}
\mathbb{E}[N_{total}] = \sum_{t=0}^{K-1} \mathbb{E}[N_t] \le \sum_{t=0}^{K-1} \left( 1 + \frac{B_v^2\mathbb{E}\bigl[\|x_t - x_0\|^2\bigr] + b_v^2}{\sigma^2} \right).
\end{equation}
Substituting $\mathbb{E}\bigl[\|x_t - x^*\|^2\bigr] \le \|x_0 - x^*\|^2 + t\eta^2\sigma^2$ and utilizing $\sum_{t=0}^{K-1} t \le \frac{K^2}{2}$ yields
\begin{equation}
\label{eq:n_total}
    \mathbb{E}[N_{total}] \le K\left( 1 + \frac{b_v^2 + 4B_v^2\|x_0 - x^*\|^2}{\sigma^2} \right) + B_v^2 K^2\eta^2.
\end{equation}

It remains to bound the last expression. We consider the two possible branches of the stepsize.

\noindent\textbf{Case 1:} $\eta = \frac{\epsilon}{2\sigma^2}$, equivalently $\epsilon \le \frac{\sigma^2}{L}$. In this case,
$$K = \left\lceil \frac{2\|x_0 - x^*\|^2}{\eta\epsilon} \right\rceil
=
\left\lceil \frac{4\sigma^2\|x_0 - x^*\|^2}{\epsilon^2} \right\rceil
=
\mathcal{O}(\epsilon^{-2}),$$
and
$$K^2\eta^2
=
\mathcal{O}(\epsilon^{-4})\cdot \mathcal{O}(\epsilon^2)
=
\mathcal{O}(\epsilon^{-2}).$$
Therefore, \eqref{eq:n_total} implies $ \mathbb{E}[N_{total}] = \mathcal{O}(\epsilon^{-2}).$

\noindent\textbf{Case 2:} $\eta = \frac{1}{2L}$, equivalently $\epsilon \ge \frac{\sigma^2}{L}$.
In this case,
$$K = \left\lceil \frac{2\|x_0 - x^*\|^2}{\eta\epsilon} \right\rceil
=
\left\lceil \frac{4L\|x_0 - x^*\|^2}{\epsilon} \right\rceil
=
\mathcal{O}(\epsilon^{-1}),$$
and 
$$K^2\eta^2
=
\mathcal{O}(\epsilon^{-2}).$$
Therefore, \eqref{eq:n_total} again implies $ \mathbb{E}[N_{total}] = \mathcal{O}(\epsilon^{-2}).$ Combining the two cases, furnishes the proof.
\end{proof}
\subsection{Proof of Theorem \ref{thm:pasta_lipschitz_convex}}\label{sec:pasta_lip_cvx}
\begin{proof}[Proof of Theorem \ref{thm:pasta_lipschitz_convex}]
The parameterization $S=1$, $p_t=1$, and $\beta_t=0$ reduces the algorithm to $x_{t+1} = x_t - \eta g_t$. By dynamic batch size definition and the conditional independence of the oracle queries
\begin{equation}
\mathbb{E}_t\bigl[\|g_t - \bar{g}_t\|^2\bigr] \le \sigma^2
\end{equation}
where $\bar{g}_t \in \partial f(x_t)$ represents the true subgradient. Expanding yields
\begin{equation}
\mathbb{E}_t\bigl[\|x_{t+1} - x^*\|^2\bigr] = \|x_t - x^*\|^2 - 2\eta\langle \bar{g}_t, x_t - x^* \rangle + \eta^2\mathbb{E}_t\bigl[\|g_t\|^2\bigr].
\end{equation}
By convexity $-\langle \bar{g}_t, x_t - x^* \rangle \le -(f(x_t) - f(x^*))$ while by $G$-Lipschitz $\|\bar{g}_t\|^2 \le G^2$ such that
\begin{equation}
\mathbb{E}_t\bigl[\|g_t\|^2\bigr] = \|\bar{g}_t\|^2 + \mathbb{E}_t\bigl[\|g_t - \bar{g}_t\|^2\bigr] \le G^2 + \sigma^2.
\end{equation}
Integrating these establishes
\begin{equation}\label{eq:cvx-lip}
\mathbb{E}_t\bigl[\|x_{t+1} - x^*\|^2\bigr] \le \|x_t - x^*\|^2 - 2\eta(f(x_t) - f(x^*)) + \eta^2(G^2 + \sigma^2).
\end{equation}
Taking the total expectation and rearranging
\begin{equation}
\mathbb{E}[f(x_t) - f(x^*)] \le \frac{\mathbb{E}\bigl[\|x_t - x^*\|^2\bigr] - \mathbb{E}\bigl[\|x_{t+1} - x^*\|^2\bigr]}{2\eta} + \frac{\eta(G^2 + \sigma^2)}{2}.
\end{equation}
Summing from $t=0$ to $K-1$ and discarding the non-positive term $\mathbb{E}\bigl[\|x_K - x^*\|^2\bigr]$ yields
\begin{equation}
\sum_{t=0}^{K-1} \mathbb{E}[f(x_t) - f(x^*)] \le \frac{\|x_0 - x^*\|^2}{2\eta} + \frac{K\eta(G^2 + \sigma^2)}{2}.
\end{equation}
Dividing  by  $K$, using convexity, and the definition of the ergodic average $\bar{x}_K$ establishes
\begin{equation}
\mathbb{E}[f(\bar{x}_K) - f(x^*)] \le \frac{\|x_0 - x^*\|^2}{2K\eta} + \frac{\eta(G^2 + \sigma^2)}{2}.
\end{equation}
Substituting  $\eta = \frac{\epsilon}{G^2 + \sigma^2}$ and $K = \left\lceil \frac{(G^2 + \sigma^2)\|x_0 - x^*\|^2}{\epsilon^2} \right\rceil$ ensures $\mathbb{E}[f(\bar{x}_K) - f(x^*)] \le \epsilon$.

To determine the total expected stochastic oracle complexity, recall \eqref{eq:cvx-lip} and remove the non-positive suboptimality term to obtain $\mathbb{E}\bigl[\|x_{t+1} - x^*\|^2\bigr] \le \mathbb{E}\bigl[\|x_t - x^*\|^2\bigr] + \eta^2(G^2 + \sigma^2)$. Unrolling this yields
\begin{equation}
\mathbb{E}\bigl[\|x_t - x^*\|^2\bigr] \le \|x_0 - x^*\|^2 + t\eta^2(G^2 + \sigma^2).
\end{equation}
Applying  $\|a + b\|^2 \le 2\|a\|^2 + 2\|b\|^2$ we get
\begin{equation}
\mathbb{E}\bigl[\|x_t - x_0\|^2\bigr] \le 4\|x_0 - x^*\|^2 + 2t\eta^2(G^2 + \sigma^2).
\end{equation}
The expected total stochastic oracle complexity amounts to
\begin{equation}
\mathbb{E}[N_{total}] = \sum_{t=0}^{K-1} \mathbb{E}[N_t] \le \sum_{t=0}^{K-1} \left( 1 + \frac{B_v^2\mathbb{E}\bigl[\|x_t - x_0\|^2\bigr] + b_v^2}{\sigma^2} \right).
\end{equation}
Substituting the bound on $\mathbb{E}\bigl[\|x_t - x_0\|^2\bigr]$ and utilizing $\sum_{t=0}^{K-1} t \le \frac{K^2}{2}$
\begin{equation}
\mathbb{E}[N_{total}] \le K\left( 1 + \frac{b_v^2 + 4B_v^2\|x_0 - x^*\|^2}{\sigma^2} \right) + \frac{B_v^2}{\sigma^2} K^2\eta^2(G^2 + \sigma^2).
\end{equation}
Substituting the definitions of $K$ and $\eta$ yields
\begin{equation}\label{eq:exact_bound_lip_cvx}
\mathbb{E}[N_{total}] \le
\left\lceil \tfrac{(G^2 + \sigma^2)\|x_0 - x^*\|^2}{\epsilon^2} \right\rceil
\left( 1 + \tfrac{b_v^2 + 4 B_v^2 \|x_0 - x^*\|^2}{\sigma^2} \right)
+
\tfrac{B_v^2\epsilon^2}{\sigma^2(G^2 + \sigma^2)}
\left\lceil \tfrac{(G^2 + \sigma^2)\|x_0 - x^*\|^2}{\epsilon^2}\right\rceil^2,
\end{equation}
which is evidently $\mathcal{O}(\epsilon^{-2})$.
\end{proof}
\subsection{Proof of Theorem \ref{thm:uncertainty_dynamic}}\label{sec:u_p}
\begin{proof}[Proof of Theorem \ref{thm:uncertainty_dynamic}]
Let $\bar{g}_t = \nabla f(x_t)$ and $g_t = \bar{g}_t + \xi_t$ for an  unbiased noise term $\xi_t$. Expanding yields
\begin{equation}
\|x_{t+1} - x^*\|^2 = \|x_t - x^* - \eta\bar{g}_t - \beta_t(x_t - x_0) - \eta\xi_t\|^2
\end{equation}
Taking the conditional expectation
\begin{equation}
\mathbb{E}_t\bigl[\|x_{t+1} - x^*\|^2\bigr] = \|x_t - x^* - \eta\bar{g}_t - \beta_t(x_t - x_0)\|^2 + \eta^2\mathbb{E}_t\bigl[\|\xi_t\|^2\bigr]
\end{equation}
Expanding establishes
\begin{equation}
\mathbb{E}_t\bigl[\|x_{t+1} - x^*\|^2\bigr] = \|x_t - x^*\|^2 - 2\langle \eta\bar{g}_t + \beta_t(x_t - x_0), x_t - x^* \rangle + \|\eta\bar{g}_t + \beta_t(x_t - x_0)\|^2 + \eta^2\mathbb{E}_t\bigl[\|\xi_t\|^2\bigr]
\end{equation}
By convexity $-2\eta\langle \bar{g}_t, x_t - x^* \rangle \le 0$. Applying the polarization identity $-2\langle a, b \rangle = \|a-b\|^2 - \|a\|^2 - \|b\|^2$ yields 
\begin{equation}
  -2\beta_t\langle x_t - x_0, x_t - x^* \rangle = -\beta_t\|x_t - x_0\|^2 - \beta_t\|x_t - x^*\|^2 + \beta_t\|x_0 - x^*\|^2.  
\end{equation}
Using $\|a + b\|^2 \le 2\|a\|^2 + 2\|b\|^2$ and the consequence of $L$-smoothness $\|\bar{g}_t\|^2 \le L^2\|x_t - x^*\|^2$ we have 
\begin{equation}
    \|\eta\bar{g}_t + \beta_t(x_t - x_0)\|^2 \le 2\eta^2 L^2\|x_t - x^*\|^2 + 2\beta_t^2\|x_t - x_0\|^2.
\end{equation}
Using BG-0 oracle and a dynamic mini-batch size $N_t$
\begin{equation}
\eta^2\mathbb{E}_t\bigl[\|\xi_t\|^2\bigr] \le \frac{\eta^2 B_v^2}{N_t}\|x_t - x_0\|^2 + \frac{\eta^2 b_v^2}{N_t}
\end{equation}
Combining these bounds yields
\begin{equation}
\begin{aligned}
\mathbb{E}_t\bigl[\|x_{t+1} - x^*\|^2\bigr] \le &\left( 1 - \beta_t + 2\eta^2 L^2 \right)\|x_t - x^*\|^2 \\
&+ \left( -\beta_t + 2\beta_t^2 + \frac{\eta^2 B_v^2}{N_t} \right)\|x_t - x_0\|^2 + \beta_t\|x_0 - x^*\|^2 + \frac{\eta^2 b_v^2}{N_t}
\end{aligned}
\end{equation}
Imposing  $\eta \le \frac{\sqrt{\beta_t}}{2 L}$  ensures the contraction coefficient $1 - \beta_t + 2\eta^2 L^2$ is bounded from above by $1 - \frac{\beta_t}{2}$. To guarantee convergence, the coefficient of $\|x_t - x_0\|^2$ must remain non-positive. We thus set $-\beta_t + 2\beta_t^2 + \frac{\eta^2 B_v^2}{N_t} \le 0$ to get $\beta_t(1 - 2\beta_t)N_t \ge \eta^2 B_v^2$, thereby leading to the residual error $\mathcal{E}_t = \beta_t\|x_0 - x^*\|^2 + \frac{\eta^2 b_v^2}{N_t}$, and concluding the proof.
\end{proof}
\subsection{Proof of Theorem \ref{thm:page_bg0_upper_bound}}\label{sec:page_bg0_upper_bound}
\begin{proof}[Proof of Theorem \ref{thm:page_bg0_upper_bound}]
Define a reference large batch size $N_{\text{ref}}$, a small batch size $b$, a fixed large-batch probability $p$, and a fixed step size $\eta$ as follows
\begin{equation} \label{eq:page_params}
    N_{\text{ref}} = \max\left\{1, \left\lceil \frac{8 B_v^2 \Delta^2}{\epsilon^4} + \frac{2 b_v^2}{\epsilon^2} \right\rceil\right\}, \quad b = \left\lceil \sqrt{N_{\text{ref}}} \right\rceil, \quad p = \frac{b}{N_{\text{ref}}}, \quad \eta = \frac{1}{4\bar{L}}.
\end{equation}
At initialization, we compute $g_0 = \frac{1}{N_0} \sum_{i=1}^{N_0} \widetilde{\nabla} f(x_0; \xi_0^{(i)})$ using $N_0 = \max\left\{1, \lceil 2 b_v^2 \epsilon^{-2} \rceil\right\}$. We set $\beta=0$ and $S=1$, thereby using the update rule $x_{t+1} = x_t - \eta g_t$. 

At each iteration $t \ge 1$, the dynamic large batch size is set to
\begin{equation}
    N_t = \max \left\{ N_{\text{ref}}, \left\lceil \frac{2(B_v^2 \|x_t - x_0\|^2 + b_v^2)}{\epsilon^2} \right\rceil \right\}.
\end{equation}
With probability $p$, the algorithm performs a reset and uses $g_t = \frac{1}{N_t} \sum_{i=1}^{N_t} \widetilde{\nabla} f(x_t; \xi_t^{(i)})$, whereas with probability $1-p$, the estimator employs the small batch size and computes $g_t = g_{t-1} + \frac{1}{b} \sum_{i=1}^b \bigl( \widetilde{\nabla} f(x_t; \zeta_t^{(i)}) - \widetilde{\nabla} f(x_{t-1}; \zeta_t^{(i)}) \bigr)$.

Define the random estimation error $V_t = \|g_t - \nabla f(x_t)\|^2$. Since $x_t$ is $\mathcal{F}_{t-1}$-measurable, the variance of the large batch is bounded by $\frac{B_v^2 \|x_t - x_0\|^2 + b_v^2}{N_t} \leq \frac{\epsilon^2}{2}$ using the definition of $N_t$. Leveraging the MSS condition and Lemma \ref{lem:page_variance}
\begin{equation} \label{eq:variance_recurrence}
    \mathbb{E}_t[V_t] \le p \frac{\epsilon^2}{2} + (1-p) \left( V_{t-1} + \frac{\bar{L}^2 \eta^2}{b} \|g_{t-1}\|^2 \right).
\end{equation}
Note any MSS function is $\bar{L}$-smoothness, such that using the descent lemma for smooth functions and the  polarization identity $-2\langle a, b \rangle = \|a-b\|^2 - \|a\|^2 - \|b\|^2$ 
\begin{equation}
    f(x_{t+1}) \le f(x_t) - \frac{\eta}{2} \|\nabla f(x_t)\|^2 - \frac{\eta}{2} (1 - \bar{L} \eta) \|g_t\|^2 + \frac{\eta}{2} V_t.
\end{equation}
We define the Lyapunov function $\Phi_t = f(x_t) - \inf_x f(x) + \frac{\eta}{2p} V_t$. Evaluating the conditional expectation of $\Phi_{t+1}$ and integrating the above descent bound with \eqref{eq:variance_recurrence}
\begin{equation} \label{eq:lyapunov_descent}
    \mathbb{E}_{t+1}[\Phi_{t+1}] \le f(x_t) - \inf_x f(x) - \frac{\eta}{2} \|\nabla f(x_t)\|^2 - \frac{\eta}{2} \left( 1 - \bar{L} \eta - \frac{(1-p)\bar{L}^2 \eta^2}{p b} \right) \|g_t\|^2 + \frac{\eta}{2p}V_t+ \frac{\eta \epsilon^2}{4}.
\end{equation}
By setting $\eta = \frac{1}{4\bar{L}}$ and utilizing $p b = \frac{b^2}{N_{\text{ref}}} \ge 1$, the coefficient of $\|g_t\|^2$ satisfies $1 - \frac{1}{4} - \frac{1-p}{16 p b} \ge \frac{3}{4} - \frac{1}{16} = \frac{11}{16} \ge \frac{1}{2}$. Taking the full unconditional expectation over all random variables yields
\begin{equation}
    \mathbb{E}[\Phi_{t+1}] \le \mathbb{E}[\Phi_t] - \frac{\eta}{2} \mathbb{E}[\|\nabla f(x_t)\|^2] - \frac{\eta}{4} \mathbb{E}[\|g_t\|^2] + \frac{\eta \epsilon^2}{4}.
\end{equation}
Telescoping this inequality  $K$ and rearranging 
\begin{equation} \label{eq:telescoping_bounds}
    \sum_{t=0}^{K-1} \mathbb{E}[\|\nabla f(x_t)\|^2] +\frac{1}{2} \sum_{t=0}^{K-1} \mathbb{E}[\|g_t\|^2]\le \frac{2 \E[\Phi_0]}{\eta} + \frac{K \epsilon^2}{2}.
\end{equation}
At initialization, we use $N_0$ oracle calls which means $\E[V_0] \le \frac{\epsilon^2}{2}$, thereby we have $\E[\Phi_0] \le \Delta + \frac{\eta \epsilon^2}{4p}$. Using $K = \lceil \frac{16 \Delta \bar{L}}{\epsilon^2} + \frac{1}{p} \rceil$ in \eqref{eq:telescoping_bounds} we have
\begin{equation} \label{eq:stationarity_bound}
    \frac{1}{K} \sum_{t=0}^{K-1} \mathbb{E}[\|\nabla f(x_t)\|^2] \le \frac{2 \Delta}{K \eta} + \frac{\epsilon^2}{2 p K} + \frac{\epsilon^2}{2} \le \frac{4 \Delta + \frac{\eta \epsilon^2}{p}}{2 \eta \left( \frac{4 \Delta}{\eta \epsilon^2} + \frac{1}{p} \right)} + \frac{\epsilon^2}{2} = \epsilon^2.
\end{equation}

To bound the dynamic batch size, using \eqref{eq:telescoping_bounds} we can write
\begin{equation}
    \sum_{t=0}^{K-1} \mathbb{E}[\|g_t\|^2] \le \frac{4 \Delta}{\eta} + \frac{\epsilon^2}{p} + K \epsilon^2 \le 16 \Delta \bar{L} + \frac{\epsilon^2}{p} + \left( \frac{16 \Delta \bar{L}}{\epsilon^2} + \frac{1}{p} + 1 \right) \epsilon^2 = 32 \Delta \bar{L} + \frac{2\epsilon^2}{p} + \epsilon^2.
\end{equation}
As we did in the proof of Theorem \ref{thm:pasta_dynamic},
\begin{equation} \label{eq:_drift}
    \mathbb{E}[\|x_t - x_0\|^2] \le t \eta^2 \sum_{j=0}^{t-1} \mathbb{E}[\|g_j\|^2] \le K \eta^2 (32 \Delta \bar{L} + \frac{2\epsilon^2}{p} + \epsilon^2).
\end{equation}
The expected computational cost at step $t \ge 1$ is $\mathbb{E}[C_t] = p \mathbb{E}[N_t] + (1-p) b \le p \mathbb{E}[N_t] + b$, which amounts to the following using the definition of our parameters
\begin{equation}
\begin{aligned}
    \mathbb{E}[C_t] &\le p \left( N_{\text{ref}} + 1 + \frac{2 B_v^2 K \eta^2 \left( 32 \Delta \bar{L} + \frac{2\epsilon^2}{p} + \epsilon^2 \right) + 2 b_v^2}{\epsilon^2} \right) + b \\
    &= 2b + p + p \frac{2 b_v^2}{\epsilon^2} + p \frac{2 B_v^2}{\epsilon^2} K \eta^2 \left( 32 \Delta \bar{L} + \frac{2\epsilon^2}{p} + \epsilon^2 \right).
\end{aligned}
\end{equation}
The total expected stochastic oracle complexity can then be evaluated as follows
\begin{equation}
    \mathbb{E}[N_{total}] \le N_0 + \sum_{t=1}^{K-1} \mathbb{E}[C_t] \le 2 b_v^2 \epsilon^{-2} + 1 + K \left( 2b + p + p \frac{2 b_v^2}{\epsilon^2} + p \frac{2 B_v^2}{\epsilon^2} K \eta^2 \left( 32 \Delta \bar{L} + \frac{2\epsilon^2}{p} + \epsilon^2 \right) \right).
\end{equation}
which, using the definition of our parameters, leads to
\begin{equation}\label{eq:exact_bound_mss}
\begin{aligned}
        \mathbb{E}[N_{total}] &\le \frac{2 b_v^2}{\epsilon^2} + 1 + \left( \frac{16 \Delta \bar{L}}{\epsilon^2} + \sqrt{\frac{8 B_v^2 \Delta^2}{\epsilon^4} + \frac{2 b_v^2}{\epsilon^2} + 1} + 1 \right) \times \\
        &\Bigg[ 2b + p + p \frac{2 b_v^2}{\epsilon^2} +  p \frac{2 B_v^2}{16 \bar{L}^2 \epsilon^2} \left( \frac{16 \Delta \bar{L}}{\epsilon^2} + \sqrt{\frac{8 B_v^2 \Delta^2}{\epsilon^4} + \frac{2 b_v^2}{\epsilon^2} + 1} + 1 \right) \left( 32 \Delta \bar{L} + \frac{2\epsilon^2}{p} + \epsilon^2 \right) \Bigg].
\end{aligned}
\end{equation}
Readily, the expected oracle complexity is $\mathcal{O}(\epsilon^{-4})$.
\end{proof}
\subsection{Proof of Theorem \ref{thm:pasta_pl}}\label{sec:pasta_pl}
\begin{proof}[Proof of Theorem \ref{thm:pasta_pl}]
We set $S=1$, $N_t = 1$, $p_t = 1$, and $\lambda = 0$ such that $x_{t+1} = x_t - \eta g_t$. By the $L$-smoothness of the objective $f$, the standard descent lemma yields
\begin{equation}
f(x_{t+1}) \le f(x_t) - \eta \langle \nabla f(x_t), g_t \rangle + \frac{L\eta^2}{2} \|g_t\|^2.
\end{equation}
Take the expectation conditioned on the current iterate $x_t$ and by the unbiasedness of the stochastic gradient oracle
\begin{equation}
\mathbb{E}_t[f(x_{t+1})] \le f(x_t) - \eta \|\nabla f(x_t)\|^2 + \frac{L\eta^2}{2} \mathbb{E}_t\bigl[\|g_t\|^2\bigr].
\end{equation}
Using the BG-0 variance condition  and unbiasedness
\begin{equation}
\mathbb{E}_t\bigl[\|g_t\|^2\bigr] \le \|\nabla f(x_t)\|^2 + B_v^2\|x_t - x_0\|^2 + b_v^2.
\end{equation}
Substituting this bound into the descent inequality
\begin{equation}
\mathbb{E}_t[f(x_{t+1})] \le f(x_t) - \eta\left(1 - \frac{L\eta}{2}\right) \|\nabla f(x_t)\|^2 + \frac{L\eta^2 B_v^2}{2}\|x_t - x_0\|^2 + \frac{L\eta^2 b_v^2}{2}.
\end{equation}
We impose $\eta \le \frac{1}{L}$ to ensure $1 - \frac{L\eta}{2} \ge \frac{1}{2}$. Subtracting $f(x^*)$ from both sides and using PL condition $-\|\nabla f(x_t)\|^2 \le -2\mu (f(x_t) - f(x^*))$ establishes
\begin{equation}
\mathbb{E}_t[f(x_{t+1}) - f(x^*)] \le (1 - \eta\mu)(f(x_t) - f(x^*)) + \frac{L\eta^2 B_v^2}{2} \|x_t - x_0\|^2 + \frac{L\eta^2 b_v^2}{2}.
\end{equation}
As we discussed in the main paper, for an $L$-smooth function, the PL condition implies QG \cite{karimi2016linear}, i.e. $\|x - x^*\|^2 \le \frac{2}{\mu}(f(x) - f(x^*))$. Expanding $\|x_t - x_0\|^2 \le 2\|x_t - x^*\|^2 + 2\|x_0 - x^*\|^2$ yields
\begin{equation}
\|x_t - x_0\|^2 \le \frac{4}{\mu}(f(x_t) - f(x^*)) + \frac{4}{\mu}(f(x_0) - f(x^*)).
\end{equation}
Substituting this bound into our previous result yields
\begin{equation}
\mathbb{E}_t[f(x_{t+1}) - f(x^*)] \le \left( 1 - \eta\mu + \frac{2L\eta^2 B_v^2}{\mu} \right)(f(x_t) - f(x^*)) + \frac{2L\eta^2 B_v^2}{\mu}(f(x_0) - f(x^*)) + \frac{L\eta^2 b_v^2}{2}.
\end{equation}
We impose $\frac{2L\eta^2 B_v^2}{\mu} \le \frac{\eta\mu}{2}$ to ensure the contraction coefficient is bounded by $1 - \frac{\eta\mu}{2}$. Unfolding recursively from $t=0$ to $K$ and applying the geometric series sum bound $\sum_{j=0}^{K-1} (1 - \frac{\eta\mu}{2})^j \le \frac{2}{\eta\mu}$ yields
\begin{equation}
\mathbb{E}[f(x_K) - f(x^*)] \le \left(1 - \frac{\eta\mu}{2}\right)^K (f(x_0) - f(x^*)) + \frac{4 L\eta B_v^2}{\mu^2}(f(x_0) - f(x^*)) + \frac{L\eta b_v^2}{\mu}.
\end{equation}
To guarantee the final suboptimality is less than $\epsilon$, we impose conditions such that each of the three terms in the bound are less than $\frac{\epsilon}{3}$. This requires $\frac{4 L\eta B_v^2}{\mu^2}(f(x_0) - f(x^*)) \le \frac{\epsilon}{3}$, $\frac{L\eta b_v^2}{\mu} \le \frac{\epsilon}{3}$, and $K \ge \frac{2}{\eta\mu} \log(\frac{3(f(x_0) - f(x^*))}{\epsilon})$. Since $S=1$ and $N_t=1$, the stochastic oracle complexity is  equivalent to $K$ and evaluates to
\begin{equation}\label{eq:exact_bound_pl}
    K = \left\lceil \max\left( \frac{2 L}{\mu} \log\left(\frac{3\Delta}{\epsilon}\right), \frac{6 L b_v^2}{\mu^2 \epsilon} \log\left(\frac{3\Delta}{\epsilon}\right), \frac{24 L B_v^2 \Delta}{\mu^3 \epsilon} \log\left(\frac{3\Delta}{\epsilon}\right) \right) \right\rceil,
\end{equation}
which is evidently $\mathcal{O}(\epsilon^{-1} \log(\epsilon^{-1}))$.
\end{proof}
\subsection{Proof of Theorem \ref{thm:pasta_star}}\label{sec:pasta_star}
\begin{proof}[Proof of Theorem \ref{thm:pasta_star}]
We set $S=1$, $N_t = 1$, $p_t = 1$, and $\lambda = 0$ such that $x_{t+1} = x_t - \eta g_t$. Expanding $\|x_{t+1} - x^*\|^2$ yields
\begin{equation}
\mathbb{E}_t\bigl[\|x_{t+1} - x^*\|^2\bigr] = \|x_t - x^*\|^2 - 2\eta\langle \nabla f(x_t), x_t - x^* \rangle + \eta^2\mathbb{E}_t\bigl[\|g_t\|^2\bigr].
\end{equation}
Applying the  definition of $\mu$-star-convexity directly bounds
\begin{equation}
-2\eta\langle \nabla f(x_t), x_t - x^* \rangle \le -2\eta(f(x_t) - f(x^*)) - \eta\mu\|x_t - x^*\|^2.
\end{equation}
Using BG-0 variance condition, unbiasedness, and $L$-smoothness implication $\|\nabla f(x_t)\|^2 \le 2L(f(x_t) - f(x^*))$ leads to
\begin{equation}
\mathbb{E}_t\bigl[\|g_t\|^2\bigr] \le 2L(f(x_t) - f(x^*)) + B_v^2\|x_t - x_0\|^2 + b_v^2.
\end{equation}
Substituting these bounds into the distance recurrence above, we have
\begin{equation}
\mathbb{E}_t\bigl[\|x_{t+1} - x^*\|^2\bigr] \le (1 - \eta\mu)\|x_t - x^*\|^2 - 2\eta\left(1 - L\eta\right)(f(x_t) - f(x^*)) + \eta^2 B_v^2\|x_t - x_0\|^2 + \eta^2 b_v^2.
\end{equation}
We impose $\eta \le \frac{1}{L}$  such that $1 - L\eta \ge 0$. Using $\|x_t - x_0\|^2 \le 2\|x_t - x^*\|^2 + 2\|x_0 - x^*\|^2$ yields
\begin{equation}
\mathbb{E}_t\bigl[\|x_{t+1} - x^*\|^2\bigr] \le \left(1 - \eta\mu + 2\eta^2 B_v^2\right)\|x_t - x^*\|^2 + 2\eta^2 B_v^2\|x_0 - x^*\|^2 + \eta^2 b_v^2.
\end{equation}
We impose $\eta \le \frac{\mu}{4 B_v^2}$ such that $2\eta^2 B_v^2 \le \frac{\eta\mu}{2}$. Thus the contraction multiplier is  bounded by $1 - \frac{\eta\mu}{2}$. Recursively unrolling from $t=0$ to $K$ and using the geometric series sum bound $\sum_{j=0}^{K-1} \left(1 - \frac{\eta\mu}{2}\right)^j \le \frac{2}{\eta\mu}$ establishes
\begin{equation}
\mathbb{E}\bigl[\|x_K - x^*\|^2\bigr] \le \left(1 - \frac{\eta\mu}{2}\right)^K \|x_0 - x^*\|^2 + \frac{4\eta B_v^2}{\mu}\|x_0 - x^*\|^2 + \frac{2\eta b_v^2}{\mu}.
\end{equation}
Using $L$-smoothness implication $f(x_t) - f(x^*) \le \frac{L}{2}\|x_t - x^*\|^2$ and $\mu$-star-convexity at  initialization $\|x_0 - x^*\|^2 \le \frac{2}{\mu}(f(x_0) - f(x^*))$ we have
\begin{equation}
\mathbb{E}[f(x_K) - f(x^*)] \le \frac{L}{\mu}\left(1 - \frac{\eta\mu}{2}\right)^K (f(x_0) - f(x^*)) + \frac{4 L\eta B_v^2}{\mu^2}(f(x_0) - f(x^*)) + \frac{L\eta b_v^2}{\mu}.
\end{equation}
To guarantee the last suboptimality is less than $\epsilon$, we ensure each term in the bound is less than $\frac{\epsilon}{3}$, i.e., $\frac{4 L\eta B_v^2}{\mu^2}(f(x_0) - f(x^*)) \le \frac{\epsilon}{3}$, $\frac{L\eta b_v^2}{\mu} \le \frac{\epsilon}{3}$, and $K \ge \frac{2}{\eta\mu} \log\left( \frac{3 L (f(x_0) - f(x^*))}{\mu \epsilon} \right)$. Since $S=1$ and $N_t=1$, the oracle complexity matches the iteration count $K$ evaluating to
\begin{equation}\label{eq:exact_bound_star}
K = \left\lceil \max\left( \frac{2 L}{\mu} \log\left( \frac{3 L \Delta}{\mu \epsilon} \right), \frac{8 B_v^2}{\mu^2} \log\left( \frac{3 L \Delta}{\mu \epsilon} \right), \frac{24 L B_v^2 \Delta}{\mu^3 \epsilon} \log\left( \frac{3 L \Delta}{\mu \epsilon} \right), \frac{6 L b_v^2}{\mu^2 \epsilon} \log\left( \frac{3 L \Delta}{\mu \epsilon} \right) \right) \right\rceil,
\end{equation}
which is readily $\mathcal{O}(\epsilon^{-1} \log(\epsilon^{-1}))$.
\end{proof}
\subsection{Proof of Theorem \ref{thm:pasta_weak}}\label{sec:pasta_weak}
\begin{proof}[Proof of Theorem \ref{thm:pasta_weak}]
Recall the strongly convex Tikhonov-regularized surrogate for epoch $s$ as $F_\lambda(x) = f(x) + \frac{\lambda}{2}\|x - x_{s-1}\|^2$ with unique minimizer $x_s^\star$. By setting $\beta_t = \lambda \eta$, the inner loop update evaluates to $x_{t+1} = x_t - \eta (g_t + \lambda(x_t - x_{s-1}))$. Applying the strong convexity of $F_\lambda$,
\begin{equation} \label{eq:inner_contraction}
\mathbb{E}_t\bigl[\|x_{t+1} - x_s^\star\|^2\bigr] \le (1 - 2\eta\mu)\|x_t - x_s^\star\|^2 + \eta^2 \mathbb{E}_t\bigl[\|g_t - \bar{g}_t\|^2\bigr] + 2\eta^2 \|\bar{g}_t\|^2 + 2\eta^2\lambda^2\|x_t - x_{s-1}\|^2
\end{equation}
where $\bar{g}_t = \mathbb{E}_t[g_t] \in \partial f(x_t)$. For $t=0$ we have $x_0 = x_{s-1}$. Substituting the variance bound from Assumption \ref{assump:oracle} and the Lipschitz bound $\|\bar{g}_t\|^2 \le G^2$ yields
\begin{equation} \label{eq:initial_contraction}
\mathbb{E}_0\bigl[\|x_1 - x_s^\star\|^2\bigr] \le (1 - 2\eta\mu)\|x_{s-1} - x_s^\star\|^2 + \frac{\eta^2}{N_0}\left(B_v^2\|x_{s-1} - x_0\|^2 + b_v^2\right) + 2\eta^2 G^2.
\end{equation}
For $t \ge 1$, expanding $\|x_t - x_0\|^2 \le 3\|x_t - x_s^\star\|^2 + 3\|x_s^\star - x_{s-1}\|^2 + 3\|x_{s-1} - x_0\|^2$ and $\|x_t - x_{s-1}\|^2 \le 2\|x_t - x_s^\star\|^2 + 2\|x_s^\star - x_{s-1}\|^2$ results in
\begin{equation} \label{eq:expanded_contraction}
\begin{aligned}
    \mathbb{E}_t\bigl[\|x_{t+1} - x_s^\star\|^2\bigr] &\le \left(1 - 2\eta\mu + \frac{3\eta^2 B_v^2}{N} + 4\eta^2\lambda^2\right)\|x_t - x_s^\star\|^2 + \left(\frac{3\eta^2 B_v^2}{N} + 4\eta^2\lambda^2\right)\|x_s^\star - x_{s-1}\|^2 \\&+ \frac{3\eta^2 B_v^2}{N}\|x_{s-1} - x_0\|^2 + \frac{\eta^2 b_v^2}{N} + 2\eta^2 G^2.
\end{aligned}
\end{equation}
The condition on $\eta$ ensures that the contraction coefficient is strictly bounded by $1 - \frac{\eta\mu}{2}$. Unrolling this recurrence from $t=1$ to $K$ for the final distance $e_s = x_K - x_s^\star$ yields
\begin{equation} \label{eq:epoch_error}
\begin{gathered}
\mathbb{E}_s\bigl[\|e_s\|^2\bigr] \le (1 - \eta\mu)^{K-1} \mathbb{E}_0\bigl[\|x_1 - x_s^\star\|^2 \bigr]+ \left(\frac{6\eta B_v^2}{\mu N} + \frac{8\eta\lambda^2}{\mu}\right)\|x_s^\star - x_{s-1}\|^2 \\ + \frac{6\eta B_v^2}{\mu N}\|x_{s-1} - x_0\|^2 + \frac{2\eta b_v^2}{\mu N} + \frac{4\eta G^2}{\mu}.
\end{gathered}
\end{equation}
Recall Lemma \ref{lem:moreau_gradient} which posits that the Moreau envelope is smooth such that using the descent lemma for it at $x_s = x_s^\star + e_s$ yields
\begin{equation} \label{eq:moreau_descent}
\mathbb{E}[\varphi_{1/\lambda}(x_s)] \le \mathbb{E}[\varphi_{1/\lambda}(x_{s-1})] - \frac{\mu}{4}\|x_{s-1} - x_s^\star\|^2 + \frac{3\lambda^2 - \rho^2}{2\mu} \mathbb{E}_s\bigl[\|e_s\|^2\bigr].
\end{equation}
Using the definitions for $K$ and $N$ alongside the substitution of \eqref{eq:epoch_error} into \eqref{eq:moreau_descent}, and applying the Moreau envelope property $\|x_{s-1} - x_s^\star\|^2 = \lambda^{-2} \|\nabla \varphi_{1/\lambda}(x_{s-1})\|^2$ results in
\begin{equation} \label{eq:stationarity_step}
\frac{\mu}{4\lambda^2} \mathbb{E}\bigl[\|\nabla \varphi_{1/\lambda}(x_{s-1})\|^2\bigr] \le \mathbb{E}[\varphi_{1/\lambda}(x_{s-1})] - \mathbb{E}[\varphi_{1/\lambda}(x_s)] + \Gamma \|x_{s-1} - x_0\|^2 + \Lambda
\end{equation}
where
\begin{equation} \label{eq:gamma_def}
\Gamma = \frac{3\lambda^2 - \rho^2}{2\mu} B_v^2 \left[ (1-\eta\mu)^{K-1} \frac{\eta^2}{N_0} + \frac{6\eta}{\mu N} \right]
\end{equation}
and
\begin{equation} \label{eq:lambda_def}
\Lambda = \frac{3\lambda^2 - \rho^2}{2\mu} \left[ (1-\eta\mu)^{K-1} \left( \frac{\eta^2 b_v^2}{N_0} + 2\eta^2 G^2 \right) + \frac{2\eta b_v^2}{\mu N} + \frac{4\eta G^2}{\mu} \right].
\end{equation}
Using a telescoping sum, we have
\begin{equation} \label{eq:telescoping_sum}
    x_{s-1} - x_0 = \sum_{i=1}^{s-1} (x_i - x_{i-1}).
\end{equation}
Using $\|\sum_{i=1}^n v_i\|^2 \le n \sum_{i=1}^n \|v_i\|^2$, yields
\begin{equation} \label{eq:cauchy_schwarz_apply}
    \|x_{s-1} - x_0\|^2 = \left\| \sum_{i=1}^{s-1} (x_i - x_{i-1}) \right\|^2 \le (s-1) \sum_{i=1}^{s-1} \|x_i - x_{i-1}\|^2.
\end{equation}
Taking the total expectation of \eqref{eq:cauchy_schwarz_apply} and summing over the epoch index $s$ from $1$ to $S$ yields
\begin{equation} \label{eq:summed_expectation}
    \sum_{s=1}^S \mathbb{E}\bigl[\|x_{s-1} - x_0\|^2\bigr] \le \sum_{s=1}^S \left( (s-1) \sum_{i=1}^{s-1} \mathbb{E}\bigl[\|x_i - x_{i-1}\|^2\bigr] \right).
\end{equation}
Using $\sum_{i=1}^{s-1} \mathbb{E}\bigl[\|x_i - x_{i-1}\|^2\bigr] \le \sum_{i=1}^S \mathbb{E}\bigl[\|x_i - x_{i-1}\|^2\bigr]$ and $ \sum_{s=1}^S (s-1) = \frac{S(S-1)}{2}$ we have
\begin{equation} \label{eq:final_cauchy_bound}
    \sum_{s=1}^S \mathbb{E}\bigl[\|x_{s-1} - x_0\|^2\bigr] \le \frac{S^2}{2} \sum_{i=1}^S \mathbb{E}\bigl[\|x_i - x_{i-1}\|^2\bigr] = \frac{S^2}{2} \sum_{s=1}^S \mathbb{E}\bigl[\|x_s - x_{s-1}\|^2\bigr].
\end{equation}
by changing the dummy variable index in the upper bound from $i$ to $s$.
Using the bound $\|x_s - x_{s-1}\|^2 \le 2\|x_{s-1} - x_s^\star\|^2 + 2\|e_s\|^2$ and summing the relation established in \eqref{eq:epoch_error} over $S$ epochs
\begin{equation} \label{eq:step_sum_bound}
\begin{aligned}
    \sum_{s=1}^S \mathbb{E}\bigl[\|x_s - x_{s-1}\|^2\bigr] &\le 2\left(1 + (1-\eta\mu)^{K-1} + \frac{6\eta B_v^2}{\mu N} + \frac{8\eta\lambda^2}{\mu}\right) \sum_{s=1}^S \mathbb{E}\bigl[\|x_{s-1} - x_s^\star\|^2\bigr] \\&+ \frac{4\mu}{3\lambda^2 - \rho^2} \Gamma \sum_{s=1}^S \mathbb{E}\bigl[\|x_{s-1} - x_0\|^2\bigr] + \frac{4\mu S}{3\lambda^2 - \rho^2} \Lambda.
\end{aligned}
\end{equation}
Substituting the sum derived from rearranging \eqref{eq:moreau_descent} leads to
\begin{equation} \label{eq:step_sum_recursive}
\begin{aligned}
    \sum_{s=1}^S \mathbb{E}\bigl[\|x_s - x_{s-1}\|^2\bigr] &\le \left( \frac{16}{\mu} + \frac{8\mu}{3\lambda^2 - \rho^2} \right) \Delta \\&+ \left( \frac{16}{\mu} + \frac{8\mu}{3\lambda^2 - \rho^2} \right) \Gamma \frac{S^2}{2} \sum_{s=1}^S \mathbb{E}\bigl[\|x_s - x_{s-1}\|^2\bigr] + \left( \frac{16}{\mu} + \frac{8\mu}{3\lambda^2 - \rho^2} \right) \Lambda S.
\end{aligned}
\end{equation}
Our batch size definitions ensure $\left( \frac{16}{\mu} + \frac{8\mu}{3\lambda^2 - \rho^2} \right) \Gamma S^2 \le 1$ such that
\begin{equation} \label{eq:isolated_bounds}
\sum_{s=1}^S \mathbb{E}\bigl[\|x_s - x_{s-1}\|^2\bigr] \le 2 \left( \frac{16}{\mu} + \frac{8\mu}{3\lambda^2 - \rho^2} \right) \Delta + 2 \left( \frac{16}{\mu} + \frac{8\mu}{3\lambda^2 - \rho^2} \right) \Lambda S
\end{equation}
which yields 
\begin{equation}
    \sum_{s=1}^S \mathbb{E}\bigl[\|x_{s-1} - x_0\|^2\bigr] \le S^2 \left( \frac{16}{\mu} + \frac{8\mu}{3\lambda^2 - \rho^2} \right) \Delta + S^3 \left( \frac{16}{\mu} + \frac{8\mu}{3\lambda^2 - \rho^2} \right) \Lambda.
\end{equation}
Substituting this back into \eqref{eq:stationarity_step} yields
\begin{equation} \label{eq:final_suboptimality}
\begin{aligned}
    \frac{1}{S}\sum_{s=1}^S \mathbb{E}\bigl[\|\nabla \varphi_{1/\lambda}(x_{s-1})\|^2\bigr] &\le \frac{4\lambda^2}{\mu}\left[ \frac{\Delta}{S} + \Gamma S \left( \frac{16}{\mu} + \frac{8\mu}{3\lambda^2 - \rho^2} \right) \Delta + \Gamma S^2 \left( \frac{16}{\mu} + \frac{8\mu}{3\lambda^2 - \rho^2} \right) \Lambda + \Lambda \right] \\&\le \frac{4\lambda^2}{\mu}\left( \frac{2\Delta}{S} + 2\Lambda \right).
\end{aligned}
\end{equation}
By our parameter selection we ensure $\frac{8\lambda^2\Delta}{\mu S} \le \frac{\epsilon^2}{2}$ and $\frac{8 \lambda^2 \Lambda}{\mu} \le \frac{\epsilon^2}{2}$, which establishes we have an $\epsilon$-stationary solution, i.e. $\frac{1}{S}\sum_{s=1}^S \mathbb{E}\bigl[\|\nabla \varphi_{1/\lambda}(x_{s-1})\|^2\bigr] \le \epsilon^2$. The total  stochastic oracle complexity can be found according to
\begin{equation} \label{eq:total_complexity}
S(N_0 + K N) \le \left( \frac{16 \lambda^2 \Delta}{\mu \epsilon^2} + 1 \right) \left[ N_0 + \left( 2 + \frac{1}{\eta\mu} \log\left( \frac{12(3\lambda^2 - \rho^2)}{\mu^2} \right) \right) N \right].
\end{equation}
Substituting the definitions for $\eta$, $N_0$, and $N$ leads to a precise, yet ugly expression (arguably, even uglier than \eqref{eq:exact_bound_mss}). Nonetheless, the oracle complexity is readily $\mathcal{O}(\epsilon^{-6})$.
\end{proof}
\end{document}